\documentclass[letterpaper,11pt]{article}

\usepackage[utf8]{inputenc} 
\usepackage[T1]{fontenc}    
\usepackage{lmodern}
\usepackage[colorlinks]{hyperref}       
\usepackage{url}            
\usepackage{booktabs}       
\usepackage{amsfonts}       
\usepackage{nicefrac}       
\usepackage{microtype}      
\usepackage[title]{appendix}
\usepackage{enumitem}

\RequirePackage{algorithm}
\RequirePackage{algorithmic}
\usepackage{amssymb}
\usepackage{graphicx}
\usepackage{fullpage}
\usepackage{comment}
\usepackage[colorlinks]{hyperref}
\usepackage{amsthm,amsmath}
\usepackage[numbers]{natbib}
\usepackage{mathtools}
\newtheorem{theorem}{Theorem}

\newtheorem{corollary}[theorem]{Corollary}

\newtheorem{lemma}[theorem]{Lemma}

\newtheorem{prop}{Proposition}
\allowdisplaybreaks

\newcommand{\name}{SubGoss}

\title{Multi-Agent Low-Dimensional Linear Bandits}

\author{Ronshee Chawla\footnote{Electrical and Computer Engineering, University of Texas at Austin. Email: ronsheechawla@utexas.edu}, \hspace{0.25mm} Abishek Sankararaman\footnote{Electrical Engineering and Computer Science Department, UC Berkeley. Email: abishek@utexas.edu} \hspace{0.5mm}and\hspace{0.5mm} Sanjay Shakkottai\footnote{Electrical and Computer Engineering, University of Texas at Austin. Email: sanjay.shakkottai@utexas.edu}}
\date{}

\begin{document}

\maketitle

\begin{abstract}
\noindent
We study a multi-agent stochastic linear bandit with side information, parameterized by an unknown vector $\theta^* \in \mathbb{R}^d$. The side information consists of a finite collection of low-dimensional subspaces, one of which contains $\theta^*$. In our setting, agents can collaborate to reduce regret by sending recommendations across a communication graph connecting them. We present a novel decentralized algorithm, where agents communicate subspace indices with each other and each agent plays a projected variant of LinUCB on the corresponding (low-dimensional) subspace. By distributing the search for the optimal subspace across users and learning of the unknown vector by each agent in the corresponding low-dimensional subspace, we show that the per-agent finite-time regret is much smaller than the case when agents do not communicate. We finally complement these results through simulations. 
\end{abstract}

\section{Introduction}
The Multi-Armed Bandit (MAB) model features a single decision maker making sequential decisions under uncertainty. It has found a wide range of applications: advertising \cite{advertising}, information retrieval \cite{info-retrival}, and operation of data-centers \cite{bandit-data-center} to name a few. See also books of \cite{lattimore_book,bubeck_book}.
As the scale of applications increases, several decision makers (a.k.a. agents) are involved in making repeated decisions as opposed to just a single agent.
For example, in internet advertising, multiple servers are typically deployed to handle the large volume of traffic \cite{buccapatnam}. 
Multi-agent MAB models have emerged as a framework to design algorithms accounting for this large scale. 

In recent times, there has been a lot of interest in the study of multi-agent unstructured bandits \cite{got,sic_mmab,aistats_gossip,colab4}. However, from a practical perspective, the linear bandit framework has shown to be more appropriate than unstructured bandits in many instances (e.g. recommendations \cite{langford-news}, clinical studies \cite{bayati}). The linear bandits framework allows for a continuum of arms with a shared reward structure, thereby 
modeling many complex online learning scenarios \cite{dani2008stochastic, oful}.
Despite its applicability, the study of multi-agent linear bandits is limited.
The key technical challenge arises from the `information leakage': the reward obtained by playing an arm gives information on the reward obtained by all other arms. In a multi-agent scenario, this is further exacerbated, making design of collaborative algorithms non-trivial. 

We take a step in this direction by considering a collaborative multi-agent low-dimensional linear bandit problem and propose a novel decentralized algorithm.
Agents in our model have side information in the form of subspaces.
In our algorithm, agents collaborate by sharing these subspaces as opposed to the linear reward in our algorithm.
Our main result shows that, even with minimal communications, the regret of all agents are much lower compared with the case of no collaboration.
\\

\textbf{Model Overview:}
Our problem consists of a single instance of a stochastic linear bandit with unknown parameter $\theta^* \in \mathbb{R}^{d}$, played concurrently by $N$ agents. 
The common side information available to the agents is a collection of $K$ disjoint $m$-dimensional subspaces, only one of which contains $\theta^*$.
However, the agents are not aware of the subspace containing $\theta^*$.
At each time $t$, each agent $i \in [N]\footnote{$[N]$ denotes the set $\{1, \cdots, N\}$}$ chooses a subspace in $[K]$.
Subsequently, it plays an action vector $a_{t}^{(i)}$ from the action set $\mathcal{A}_{t} \subset \mathbb{R}^{d}$ while satisfying the constraints imposed by the chosen subspace and receives a reward $\langle a_t^{(i)},\theta^* \rangle + \eta_{t}^{(i)}$, where $\eta_{t}^{(i)}$ is zero mean sub-gaussian noise. 
Thus, the above problem can be visualized as a two-tier bandit problem, described as follows: 
the first tier corresponds to the $K$ arms of an unstructured bandit.
In second tier, each arm corresponds to solving the stochastic linear bandit problem (parametrized by unknown $\theta^*$) over one of the $K$ known subspaces.
The rewards obtained by the agents are only dependent on their actions and independent of actions of other agents.

The agents in our model are connected through a communication graph over which they can exchange messages to collaborate. 
Agents are constrained to communicate by exchanging messages for a fixed number of times over any given time span. 
We seek decentralized algorithms for agents, i.e., the choice of action vector, communication choices and messages depend only on the observed past history (of action vectors, rewards, and messages) of that agent. 
\\

\textbf{Motivating Example:}
We motivate our model in the context of personalized news recommendation systems. Suppose that a user $u$ can be modeled by an (unknown) vector $\theta^*_u \in \mathbb{R}^d$, which lies in one of $K$ possible subspaces. These subspaces reflect information from historical data of other users whose feature vectors have been learned (for example, users that have been in the system for a long time) and can be categorized into a collection of low-dimensional subspaces. Thus, when any new user enters the system, it needs to {\em (i)} identify the subspace that the user’s vector lies in, and {\em (ii)} determine the corresponding $\theta^*_u$ for that user. At any point of time, each user is handled by a single server (agent). However, in large scale applications, a collection of servers are deployed to handle the large number of users. Even though a user's queries are routed to different servers over time; however, these servers can collaborate by exchanging messages to speed up learning.

Elaborating on the news recommendation example above, the subspaces could correspond to political leanings of the user (e.g. social liberal, fiscal conservative, libertarian, etc.). In this model, all users with the same political leanings would share the same subspace, however, their personal vectors $\theta^*_u$ would differ (to capture fine-grained individual differences). The two-tier bandit thus models a coarse dimensionality reduction through the subspace choice and a finer characterization through  $\theta^*_u$ in a specific low-dimensional subspace.

The above discussion reflects the two-tier bandit from a single user's perspective; the system will run many parallel instances of this -- one for each user. Our model abstracts this setup and our algorithm demonstrates that agents (servers) can indeed benefit from collaborations with minimal communications.
\\

Our \textbf{main contributions} are as follows:
\\

{\bf 1. The \name \ Algorithm:} We propose \name \ (Algorithm \ref{algo:main-algo}), which proceeds in phases, such that agents in any phase (a) explore the subspaces repeatedly to identify the correct subspace containing $\theta^*$, followed by (b) playing Projected LinUCB on that subspace, and (c) communicating that subspace whenever requested.
Our algorithm constrains agents to search for $\theta^{*}$ over only a small set (of cardinality $\leq \frac{K}{N}+2$) of subspaces per agent.
Agents use pairwise communications to recommend subspaces (not samples), i.e., agents communicate the ID of the estimated best subspace. 
This set of subspaces is updated through recommendations: agents accept new recommendations and drop subspace(s) unlikely to contain $\theta^{*}$, ensuring that the total number of subspaces an agent considers at all times remain small.
Agents can communicate $O(\log T)$ times over a span of time $T$.
Nevertheless, the best one \emph{spreads} to all the agents through communications and thus all agents eventually identify the correct subspace.
\\

\textbf {2. Asymptotically matching an oracle's regret rate in large systems:} Despite playing from a time-varying set of subspaces, every agent incurs a regret of at most  $O(m\sqrt{T} \log T) + O \left( \frac{K}{N}.m\sqrt{T} \right)$ (Theorem \ref{thm:mainresult}). 
This scaling \emph{does not depend}\footnote{We require $G$ to be connected. See Appendix \ref{sec:assumption_discussion}} on the gossip matrix $G$ and we show that
these communication constraints only affect the constant term in regret. 

Note that an oracle that knew the right subspace will only incur regret for finding $\theta^*$ in that subspace, while avoiding any regret due to subspace search.
We use this fact at the end of Section \ref{subgossreg} to informally argue that even if an agent gets the information about the correct subspace whenever it communicates with other agents, it cannot do better than $\Omega(m\sqrt{T})$ regret under our model of information sharing.  
Consequently, we show in Corollary \ref{cor:collabbenefit} that for large $K$ and $N$, \name\ achieves near-optimal performance, demonstrating that it uses the side information effectively. \\


\textbf{3. Finite-time gains due to faster search of subspaces with collaboration:} We quantify the extent to which collaboration helps by analyzing the ratio of regret upper bound achieved by \name \ without collaborations to that achieved with collaboration. We observe that the benefits occur from the ability of multiple agents to do a faster search for the right subspace containing $\theta^*$ as compared to a single agent.

In high dimensional settings (when $d$ is large and $m$ is a constant) with large number of subspaces and agents ($K=N=O(d)$), we show in Corollary \ref{cor:collabbenefit} (and the remarks following it), that by time $T=\Omega(d)$, the collaborative gain is of the order of $\Omega\left(\frac{d}{\log d}\right)$. The key reason for the gain lies in the ability of multiple agents to distribute the search for the right subspace among them, enabling \emph{all} agents to identify the subspace faster, compared to a single agent without collaboration. Finally, these results are corroborated through simulations (Fig. \ref{fig1}). \\

\noindent {\bf{Related Work:}} Our work focuses on collaborative multi-agent bandits,
where agents jointly accomplish the shared objective of minimizing cumulative regret \cite{sigmetrics,colab1,buccapatnam,colab3,colab4,sigmetrics, aistats_gossip, vial2020robust,proutiere}. Our work focuses on a setting where agents only share recommendations of actions (e.g., to minimize network traffic due to collaboration, while optimizing for per-agent regret), and do not share the samples themselves \cite{sigmetrics, aistats_gossip, vial2020robust}. In each of these studies, agents play from a small subset of arms at all times and exchange the arm IDs of what they consider the best arm in their playing set through pairwise gossip-style communications, which is further used to update their playing set.  
Another approach that focuses on reducing network traffic, while simultaneously optimizing for total cumulative regret (sum over time and users) is based on the follow-the-leader approach \cite{proutiere}, wherein a leader among the agents is elected and subsequently becomes the sole player exploring the arms, while other agents act as its followers. 
However, all of the above works are adapted to the case of finite-armed unstructured MABs and cannot be applied to a linear bandit setup such as ours.
Nevertheless, we adopt some of the broader principles from \cite{sigmetrics, aistats_gossip} regarding the use of the gossiping paradigm for communications to spread the best subspace into our algorithm design.

The stochastic linear bandit framework and the study of LinUCB algorithm was initiated by \cite{dani2008stochastic,oful}.  
From a practical perspective, the linear bandit framework has been shown to be effective for various applications: for example \cite{practical_agarwal,langford-news} apply this framework in the context of internet advertising and \cite{bayati,contextual_health} apply in the context of clinical trials. 
Furthermore, a projected version of LinUCB on low-dimensional subspaces has been recently studied in \cite{hassibi_PSLB}.

To the best of our knowledge, our model has not been studied before, even in a single-agent setting. Our model can be viewed as a generalization of the well-studied model of sparse linear bandits \cite{bayati}, \cite{carpentiersparsebandit, gerchinovitz2011sparsity, abbasisparseonline}. The sparse linear bandit problem assumes that the unknown vector $\theta^*$ is $s$-sparse, for some known sparsity level $s < d$. In other words, $\theta^*$ is assumed to lie in one of the ${d \choose s}$ subspaces, where each of these subspaces corresponds to a particular set of $s$ coordinates, i.e., the subspaces are axis aligned. Our model is a generalization where $\theta^*$ lies in one of any $K$ given arbitrary disjoint subspaces. 
The two main algorithmic ideas in sparse bandits are to either use heavy-tailed priors for sampling action vectors and the associated posterior distributions that result in sparse estimates \cite{gerchinovitz2011sparsity, abbasisparseonline}, or use a LASSO-type regularizer in the estimator \cite{bayati}.  We cannot use the techniques from sparse linear bandits in our model because even though the unknown $\theta^*$ lies in one of the low-dimensional subspaces, all of its $d$ coordinates can have non-zero values.
Consequently, the linear bandit suffers from the problem of ``information sharing'': the reward obtained by playing an action vector in one subspace reveals information about the rewards of action vectors in other subspaces.
Hence, algorithmic ideas from sparse bandits are not directly applicable in our setting. 

The study of multi-agent linear bandit framework has attracted a lot of attention lately \cite{Wang2020Distributed}. 
Multi-agent linear bandits have been studied in the context of clustering \cite{pmlr-v48-korda16}, differentially-private federated learning \cite{dubey2020differentiallyprivate}, and safety-critical distributed learning \cite{amani2020decentralized}.
However, all of these works involve agents sharing samples with each other in the absence of side information, unlike our setting where agents have the side information in the form of subspaces and communicate only subspace IDs with each other.

\section{Problem Setup}

Our problem setup consists of single instance of stochastic linear bandit (parametrized by unknown $\theta^*$), concurrently played by $N$ agents.
All agents play from the same set of action vectors $\{\mathcal{A}_t\}_{t \in \mathbb{N}}$ at any time $t$, where $\mathcal{A}_t \subset \mathbb{R}^{d}$. 
The side information available to all the agents is a collection of $K$ disjoint subspaces in $\mathbb{R}^{d}$ of dimension $m < d$. These subspaces are denoted by the $d \times m$ orthonormal matrices $\{U_{i}\}_{i=1}^{K}$, where $\mathrm{span}(U_i)$ defines a $m$-dimensional subspace in $\mathbb{R}^{d}$. One of these subspaces contains $\theta^{*}$, but agents are unaware of the subspace containing it. Without loss of generality, we assume that $\|\theta^{*}\|_{2} \leq 1$ and $K$ is an integral multiple of $N$. Let $P_{k} = U_{k} U_{k}^{T}$ denote the projection matrix of the subspace $\mathrm{span}(U_k)$ for all $k \in [K]$. We also assume that the set of action vectors $\{\mathcal{A}_t\}_{t \in \mathbb{N}}$ contain the orthonormal basis vectors of all the $K$ subspaces (which are columns of $U_k$ for all $k \in [K]$) for all $t \in \mathbb{N}$. 

At any time $t$, an agent $i$ chooses a subspace in $[K]$. Subsequently, it plays an action vector $a_{t}^{(i)} \in \mathcal{A}_{t}$ while satisfying the constraints imposed by the chosen subspace and the reward obtained is given by $r_t^{(i)}:= \langle a_{t}^{(i)},\theta^{*}\rangle + \eta_t^{(i)}$. Here, $\eta_{t}^{(i)}$ is a zero mean sub-gaussian noise, conditional on the actions and rewards accumulated only by agent $i$, i.e., for all $z \in \mathbb{R}$, 
$\mathbb{E}[\exp(z\eta_{t}^{(i)})|\mathcal{F}_{t-1}^{(i)}] \leq \exp\left(\frac{z^2}{2}\right)$ a.s. and $\mathcal{F}_{t-1}^{(i)} = \sigma(a_{1}^{(i)}, r_{1}^{(i)}, \cdots, a_{t-1}^{(i)}, r_{t-1}^{(i)}, a_{t}^{(i)})$. The noise is independent across agents.
Thus, the above setup can be abstracted as a two-tier bandit problem, where: \emph{(a)} the first tier corresponds to the $K$ arms of an unstructured bandit, and \emph{(b)} in second tier, each arm corresponds to solving the stochastic linear bandit problem over one of the $K$ subspaces.
\\


\textbf{Collaboration among Agents:}
Our model builds on gossip-based communication constraints for multi-agent finite-armed unstructured bandits in \cite{sigmetrics, aistats_gossip}. The agents collaborate by exchanging messages over a communication network. This matrix is represented through a $N \times N$ \emph{gossip matrix} $G$, with rows in this matrix being probability distributions over $[N]$.
At each time step, after playing an action vector and obtaining a reward, agents can additionally choose to communicate with each other.
Agent $i$, if it chooses to communicate, will do so with another agent $j \sim G(i,\cdot)$, chosen independently of everything else.  
However, for any time horizon $T$, the total number of times an agent can communicate is $O(\log T)$. 
Each time an agent chooses to communicate, it can exchange at most $\log_{2} K + 1$ number of bits. 
Therefore, every agent communicates $O(\log K . \log T)$ bits over a time horizon of $t$, for all $t$.
\\

\textbf{Decentralized Algorithm:} Each agent's decisions (arm play and communication decisions) in the algorithm depend only on its own history of plays and observations, along with the the recommendation messages that it has received from others.
\\

\textbf{Performance Metric:}
Each agent plays action vectors in order to minimize their individual cumulative regret. 
At any time $t$, the instantaneous regret for an agent $i$ is given by $w_{t}^{(i)} = \langle\theta^{*}, a_{t}^{*}\rangle - \langle\theta^{*}, a_{t}^{(i)}\rangle$, where $a_{t}^{*} = \arg \max_{a \in \mathcal{A}_t} \langle\theta^{*}, a\rangle$.
The expected cumulative regret for any agent $i \in [N]$ is given by $\mathbb{E}[R_{T}^{(i)}] \coloneqq \mathbb{E}[\sum_{t=1}^{T}w_{t}^{(i)}]$, where the expectation is with respect to the $\sigma$-field generated by action vectors and rewards of agent $i$ up to and including time $T$.

\section{\name \ Algorithm}
\label{algo}
\textbf{Key Ideas and Intuition:} Our setting considers that the unknown $\theta^*$ lies in one of a large number of (low-dimensional) subspaces. 
In our approach, agents at any time instant identify a small {\em active set of subspaces} (cardinality $\leq K/N + 2$) and play actions only within this set of subspaces (however, this set is time-varying).
At each point of time, an agent first \emph{identifies} among its current active set of subspaces the one likely to contain $\theta^*$. It subsequently plays a projected version of LinUCB on this identified subspace. 
The communication medium is used for \emph{recommendations}; whenever an agent is asked for information, it sends as message the subspace ID that it thinks most likely contains $\theta^*$, which is then used to update the active set of the receiving agent. Thus, an agent's algorithm has two time-interleaved parts: (a) updating active sets through collaboration, which is akin to a distributed best-arm identification problem, and (b) determining the optimal $\theta^*$ from within its active set of subspaces, an estimation problem in low dimensions, similar to the classical linear bandit.

\name \ algorithm is organized in phases
with the active subspaces for each agent fixed during the phase.
Within a phase, all agents solve two tasks - {\em (i)} identify the most likely subspace among its active subspaces to contain $\theta^*$ and {\em (ii)} within this subspace, play actions optimally to minimize regret. 
The first point is accomplished by agents through pure exploration. 
In pure exploration, agents play the orthonormal basis vectors of all the subspaces in their respective active sets in a round-robin fashion. Agents minimize their regret during pure exploration by considering a \emph{small} active set of subspaces (of cardinality $\leq K/N + 2$) at all times. 
Otherwise, agents play action vectors within their best estimated subspace containing $\theta^*$ to minimize regret. This step is achieved by playing a projected version of the LinUCB algorithm.
The second step only incurs regret in the dimension of the subspace (once the true subspace is correctly identified) as opposed to the ambient dimension, thereby keeping regret low.
Due to communications, the correct subspace \emph{spreads} to all agents, while playing from a small active set of active subspaces at all times (and thus reducing the regret due to explorations).

\subsection{Description}
\name \ algorithm builds on some of the ideas developed for a (non-contextual) collaborative setting for unstructured bandits in \cite{aistats_gossip}. We fix an agent $i \in [N]$ for ease of exposition. \name \ proceeds in phases, where phase $j \in \mathbb{N}$ is from time slots $\sum_{l=1}^{j-1} \lceil b^{l-1} \rceil + 1$ to $\sum_{l=1}^{j} \lceil b^{l-1} \rceil$, both inclusive, where $b > 1$ is given as an input.
During each phase $j$, agent $i$ only plays from an active set $S_j^{(i)} \subset[K]$ of subspaces such that $|S_j^{(i)}| \leq (K/N)+2$.
Agents communicate at the end of the phase to update their active set.
    Notice that the phase length is $\lceil b^{j-1} \rceil$, which satisfies the communication constraint of $O(\log T)$ communications for any time horizon $T$.
\\

{\bf Initialization:} At the beginning of the algorithm, every agent is assigned a \emph{sticky} set of subspaces, by partitioning the subspaces equally across agents:
\begin{equation}
\label{eqn:hat_S_i}
        (\widehat{S}^{(i)})_{i=1}^N = \left\{(i-1)\frac{K}{N} + 1, \cdots, i\frac{K}{N}\right\}. 
\end{equation}
We set the initial active set ${S}_1^{(i)} = \widehat{S}^{(i)}$.
\\

{\bf Action Vectors Chosen in a Phase:} We play the following two subroutines in every phase $j \in \mathbb{N}$ in the order as described below:

1. {\sc{Explore}} - In this subroutine, for every $k \in S_{j}^{(i)}$, we play the orthonormal basis vectors of the subspace $\mathrm{span}(U_k)$ (which are the columns of $U_k$) in a round robin fashion for $8m\lceil b^{\frac{j-1}{2}} \rceil$ times. 

Let $\widetilde{n}_{k,j}^{(i)}$ denote the number of times subspace $\mathrm{span}(U_k)$ has been explored by agent $i$ up to and including phase $j$. 
After executing the explore subroutine in a phase, agent $i$ calculates the least squares estimates $\widetilde{\theta}_{k,j}^{(i)}$ for every $k \in S_{j}^{(i)}$ by using only the explore samples of the subspace $\mathrm{span}(U_k)$ up to and including phase $j$. Mathematically, $\widetilde{\theta}_{k,j}^{(i)} = \arg \min_{\theta \in \mathbb{R}^{d}} \|(\widetilde{A}_{k,\widetilde{n}_{k,j}^{(i)}}^{(i)})^{T} \theta - \widetilde{\mathbf{r}}_{k,\widetilde{n}_{k,j}^{(i)}}^{(i)}\|_{2}$, where $\widetilde{A}_{k,\widetilde{n}_{k,j}^{(i)}}^{(i)}$ is a $d \times \widetilde{n}_{k,j}^{(i)}$ matrix whose columns are the explore action vectors of the subspace $\mathrm{span}(U_k)$ played up to and including phase $j$ and $\widetilde{\mathbf{r}}_{k,\widetilde{n}_{k,j}^{(i)}}^{(i)}$ is a column vector of the corresponding rewards. It is worth noticing that $\widetilde{\theta}_{k,j}^{(i)}$ is the estimate of the vector $P_{k}\theta^{*}$ (details in the proof of Lemma \ref{lemma:intstep}), which is the projection of the unknown vector $\theta^*$ in the subspace $\mathrm{span}(U_k)$. We will describe in the proof sketch (Section \ref{proofsketch}) that this observation is crucial to finding the subspace containing $\theta^*$.

2. {\sc{Projected LinUCB}} - Let $\widehat{\mathcal{O}}_{j}^{(i)} = \arg \max_{k \in S_{j}^{(i)}} \|\widetilde{\theta}_{k,j}^{(i)}\|_2$. For the remainder of the phase $j$, agent $i$ chooses the action vector according to the Projected LinUCB \cite{hassibi_PSLB}, played on the subspace $\mathrm{span}(U_{\widehat{\mathcal{O}}_{j}^{(i)}})$.
We set $k = \widehat{\mathcal{O}}_{j}^{(i)}$ for reducing the clutter while describing Projected LinUCB.
For all $\sum_{l=1}^{j-1} \lceil b^{l-1} \rceil+8m|S_{j}^{(i)}| \lceil b^{\frac{j-1}{2}}\rceil < t \leq \sum_{l=1}^{j} \lceil b^{l-1} \rceil$, where $t$ denotes the corresponding time instants after the end of explore subroutine in phase $j$, let $n_{k,t}^{(i)}$ denote the number of times agent $i$ has played Projected LinUCB on the subspace $\mathrm{span}(U_k)$ up to time $t$.
The action vector chosen is given according to the following equations \cite{hassibi_PSLB}: 
\begin{align*}
a_{t}^{(i)} &\in \arg \max_{a \in \mathcal{A}_{t}} \max_{\theta \in \mathcal{C}_{k,t}^{(i)}} \langle \theta, P_{k}a \rangle, \text{where} \\
\mathcal{C}_{k,t}^{(i)} &= \left\{\theta \in \mathbb{R}^{d}: ||\widehat{\theta}_{t}^{(i)}-\theta||_{\bar{V}_{k,t}(\lambda)^{(i)}} \leq \beta_{t,\delta}\right\}, 
\end{align*}
$\beta_{t,\delta} = \sqrt{\lambda} + \sqrt{2 \log{\frac{1}{\delta}} + m \log \left(1 + \frac{n_{k,t}^{(i)}}{\lambda m}\right)}$, $\bar{V}_{k,t}(\lambda)^{(i)} = P_{k} (A_{k,t-1}^{(i)} (A_{k,t-1}^{(i)})^{T} + \lambda I_{d}) P_{k}$, and $\widehat{\theta}_{k,t-1}^{(i)} = \arg \min_{\theta \in \mathbb{R}^{d}} ||(P_{k}A_{k,t-1}^{(i)})^{T} \theta - \mathbf{r}_{k,t-1}^{(i)}||_{2}^{2} + \lambda ||P_{k}\theta||_{2}^{2}$. $A_{k,t-1}^{(i)}$ is a $d \times n_{k,t}^{(i)}$ matrix whose columns are the Projected LinUCB action vectors played only on the subspace $\mathrm{span}(U_k)$ up to time $t$ and $\mathbf{r}_{k,t-1}^{(i)}$ is a column vector of the corresponding rewards.
\\

{\bf Communications and the Active Subspaces for the Next Phase:} After phase $j$ gets over, agent $i$ asks for a subspace recommendation from an agent $J \sim G(i,\cdot)$ chosen independently. Denote by $\mathcal{O}_j^{(i)} \in [K]$ to be this recommendation. Agent $i$ if asked for a recommendation at the end of phase $j$, recommends the subspace ID $\widehat{\mathcal{O}}_{j}^{(i)}$, i.e., using only the explore samples.
The next active set is constructed as follows: (i) if $\mathcal{O}_{j}^{(i)} \in S_{j}^{(i)}$, the active set remains unchanged, (ii) if $\mathcal{O}_{j}^{(i)} \notin S_{j}^{(i)}$ and $|S_{j}^{(i)}| < \frac{K}{N}+2$, then $S_{j+1}^{(i)} \coloneqq S_{j}^{(i)} \cup \mathcal{O}_{j}^{(i)}$, and (iii) if $\mathcal{O}_{j}^{(i)} \notin S_{j}^{(i)}$ and $|S_{j}^{(i)}| = \frac{K}{N}+2$, then $S_{j+1}^{(i)} \coloneqq \widehat{S}^{(i)} \cup \mathcal{B}_{j}^{(i)} \cup \mathcal{O}_{j}^{(i)}$, where $\mathcal{B}_{j}^{(i)} = \arg \max_{k \in S_{j}^{(i)} \backslash \widehat{S}^{(i)}} \|\widetilde{\theta}_{k,j}^{(i)}\|_2$.
Observe that $\widehat{S}^{(i)} \subseteq S_j^{(i)} \forall j \in \mathbb{N}$, and thus, $\widehat{S}^{(i)}$ is called sticky. Moreover, the update step along with the initialization $S_{1}^{(i)} = \widehat{S}^{(i)}$ also ensures that $|S_{j}^{(i)}| \leq \frac{K}{N}+2$ for all phases $j \in \mathbb{N}$. 
\\

Please see Algorithm \ref{algo:main-algo} for the pseudo-code of the \name \ Algorithm.
\\

\begin{algorithm}[t]
   \caption{\name \ Algorithm (at Agent $i$)}
   \label{algo:main-algo}
\begin{algorithmic}[1]
   \STATE {\bfseries Input:} $K$ disjoint $m$-dimensional subspaces $\{U_{l}\}_{l=1}^{K}$, $b>1$, regularization parameter $\lambda > 0$, $\delta \in (0, 1)$.
   \STATE {\bfseries Initialization}:  $ \widehat{S}^{(i)},S_1^{(i)}$ (Equation (\ref{eqn:hat_S_i})), $j \gets 1$.
   \WHILE{phase $j \geq 1$}
   \STATE {\sc Explore:} For each $k \in S_{j}^{(i)}$, play the orthonormal basis vectors of the subspace ID $k$ in a round robin fashion for $8m\lceil b^{\frac{j-1}{2}} \rceil$ times.
   \STATE Calculate the least squares estimate $\widetilde{\theta}_{k,j}^{(i)}$ for each $k \in S_{j}^{(i)}$ after running the {\sc Explore} by using only its explore samples collected thus far.
   \STATE $\widehat{\mathcal{O}}_{j}^{(i)} \gets \arg \max_{k \in S_{j}^{(i)}} \|\widetilde{\theta}_{k,j}^{(i)}\|_2$.
   \STATE {\sc Projected LinUCB:} For the remainder of the phase $j$, play the Projected LinUCB \cite{hassibi_PSLB} on the subspace ID $\widehat{\mathcal{O}}_{j}^{(i)}$ by using only its Projected LinUCB samples collected thus far.
   \STATE At the end of phase $j$, sample an agent from the gossip matrix $ag \sim G(i,\cdot)$ for receiving subspace recommendation. 
   \STATE Get the subspace recommendation $\mathcal{O}_j^{(i)} \gets \arg \max_{k \in S_{j}^{(ag)}} \|\widetilde{\theta}_{k,j}^{(ag)}\|_2$.
   \STATE Active set update for the next phase:
   \IF{$\mathcal{O}_{j}^{(i)} \in S_{j}^{(i)}$}
   \STATE $S_{j+1}^{(i)} \gets S_{j}^{(i)}$.
   \ELSE
   \IF{$|S_{j}^{(i)}|<\frac{K}{N}+2$}
   \STATE $S_{j+1}^{(i)} \gets S_{j}^{(i)} \cup \mathcal{O}_{j}^{(i)}$.
   \ELSIF{$|S_{j}^{(i)}|=\frac{K}{N}+2$}
   \STATE $\mathcal{B}_{j}^{(i)} \gets \arg \max_{k \in S_{j}^{(i)} \backslash \widehat{S}^{(i)}} \|\widetilde{\theta}_{k,j}^{(i)}\|_2$.
    \STATE $S_{j+1}^{(i)} \gets \widehat{S}^{(i)} \cup \mathcal{B}_{j}^{(i)} \cup \mathcal{O}_{j}^{(i)}$.
    \ENDIF
   \ENDIF
	\STATE $j \gets j+1$.
   \ENDWHILE
\end{algorithmic}
\end{algorithm}

\noindent {\bf{Remarks}:} 

1. Until phase $\tau_0$ (defined in Theorem \ref{thm:mainresult}), 
the duration of the explore subroutine exceeds $\lceil b^{j-1} \rceil$. In order to make less noisy subspace recommendations until phase $\tau_0$, the exploration is equally distributed across all the subspaces in $S_{j}^{(i)}$ for the entire duration of the phase.

2. {\bf Choice of $\mathcal{C}_{k,t}^{(i)}$ while playing Projected LinUCB -} The construction and analysis of the confidence set $\mathcal{C}_{k,t}^{(i)}$ is formally described in Theorem \ref{ConfidenceSet} in the appendix.
The confidence set is an ellipsoid in the subspace on which Projected LinUCB is played.
It is constructed such that: (a) it contains $\theta^*$ with high probability, and (b) it shrinks in size as the correct sequence of action vectors is played with time.
\\

3. {\bf Choice of $a_t^{(i)}$ and its computational complexity while playing Projected LinUCB -} Analogous to the upper confidence bound (UCB) for classical $K$-armed bandits, an agent playing Projected LinUCB calculates an upper bound for the reward obtained for every $a \in \mathcal{A}_{t}$ and plays the action vector that maximizes the upper bound.
This can be observed for the case when the action set $\mathcal{A}_{t}$ is finite, as follows: for a fixed $a \in \mathcal{A}_{t}$
\begin{align*}
    \langle \theta, P_{k}a \rangle &= \langle \theta-\widehat{\theta}_{k,t-1}^{(i)}, P_{k}a \rangle + \langle \widehat{\theta}_{k,t-1}^{(i)}, P_{k}a \rangle,\\
    &\leq \|\theta-\widehat{\theta}_{k,t-1}^{(i)}\|_{\bar{V}_{k,t}(\lambda)^{(i)}}.\|P_{k}a\|_{(\bar{V}_{k,t}(\lambda)^{(i)})^{\dagger}} + \langle \widehat{\theta}_{k,t-1}^{(i)}, P_{k}a \rangle,\\
    &\leq \beta_{t,\delta} \|P_{k}a\|_{(\bar{V}_{k,t}(\lambda)^{(i)})^{\dagger}} + \langle \widehat{\theta}_{k,t-1}^{(i)}, P_{k}a \rangle,
\end{align*}
where the first inequality is obtained by applying H\"{o}lder's inequality and the second inequality follows by using the definition of $\mathcal{C}_{k,t}^{(i)}$.
Therefore, for finite action sets, Projected LinUCB plays the action vector
\begin{align}
\label{projlinucb:exact}
a_{t}^{(i)} = \arg \max_{a \in \mathcal{A}_{t}} \langle \widehat{\theta}_{k,t-1}^{(i)}, P_{k}a \rangle +  \beta_{t,\delta} \|P_{k}a\|_{(\bar{V}_{k,t}(\lambda)^{(i)})^{\dagger}}.
\end{align}
The first term is the empirical estimate of the reward of the action $a$ and the second term corresponds to the deviation around that estimate, similar to the UCB value in $K$-armed bandits.
The computational complexity of determining $a_{t}^{(i)}$ depends on the computational complexity of calculating $(\bar{V}_{k,t}(\lambda)^{(i)})^{\dagger}$, $\widehat{\theta}_{k,t-1}^{(i)}$ and the inner products in equation (\ref{projlinucb:exact}).


\section{Main Result}
\label{subgossreg}
In order to state the result, we assume that the gossip matrix $G$ is connected (detailed definition in Appendix \ref{sec:assumption_discussion}).
We define a random variable $\tau_{\mathrm{spr}}^{(G)}$ denoting
the spreading time of the following process: node $i$ initially has a rumor; at each time, an agent without a rumor calls another chosen independently from the gossip matrix $G$ and learns the rumor if the other agent knows the rumor. The stopping time $\tau^{(G;i)}_{\mathrm{spr}}$ denotes the first time when agent $i$ knows the rumor and $\tau_{\mathrm{spr}}^{(G)} = \max_{i \in [N]}\tau_{\mathrm{spr}}^{(G;i)}$ is the time by which all agents know the rumor.
For ease of exposition, we assume that $\theta^* \in \textrm{span}(U_1)$, which the agents are unaware of and $\|a\|_2 \leq 1$ for all $a \in \cup_{t=1}^{T} \mathcal{A}_{t}$. Let $\Delta = \min_{k \in [K]: k \neq 1} \Delta_k$, where $\Delta_k = \|P_{1}\theta^{*}-P_{k}\theta^{*}\|_2$.

\begin{theorem}
\label{thm:mainresult}
Consider a system consisting of $N$ agents connected by a gossip matrix $G$, all running \name \ Algorithm with $K$ disjoint $m$-dimensional subspaces and input parameters $b > 1$, $\lambda \geq 1$, $\delta = \frac{1}{T}$.
Then, the expected cumulative regret of any agent $i \in [N]$, after time $T \in \mathbb{N}$ is bounded by:
	\begin{multline}
	\label{eqn:per-agent-regret-thm}
	\mathbb{E}[R_T^{(i)}] \leq  \underbrace{\sqrt{8mT\beta_{T}^{2}\log\left(1+\frac{T}{m\lambda}\right)}+2}_{\text{{Projected LinUCB Regret}}}+ \underbrace{2g(b)\left(\lceil b^{2\tau_0} \rceil + \frac{48b^3}{\log b}.\frac{m^{4}N}{\Delta^6}+b\mathbb{E}[b^{2 \tau_{\mathrm{spr}}^{(G)}}]\right)}_{\text{{Constant Cost of Pairwise Communications}}} +\\ \underbrace{16 m \left(\frac{K}{N}+2\right) \log_{b}(h_{b, T}) 
+ 16 m \left(\frac{K}{N}+2\right) \frac{\sqrt{h_{b, T}}-1}{\sqrt{b}-1}}_{\text{Cost of subspace exploration}},
	\end{multline}
	where $\beta_{T} = \sqrt{\lambda} + \sqrt{2 \log T + m \log \left(1 + \frac{T-1}{\lambda m}\right)}$, $g(b) = \left(\frac{1}{b-1}+\frac{1}{\log b}\right)$, $h_{b,T} = b(1+(T-1)(b-1))$, and $\tau_0 = \min\left\{j \in \mathbb{N}: \forall j' \geq j, \lceil b^{j'-1} \rceil \geq 8m\left(\frac{K}{N}+2\right)\lceil b^{\frac{j'-1}{2}} \rceil\right\}$, 
\end{theorem} 

{\bf{Remarks:}}

1. Proposition \ref{prop:tau0bound} shows that $\tau_0 \leq 2 \log_{b}\left(16m\left(\frac{K}{N}+2\right)\right)+1$ and thus, the term $b^{2\tau_0}$ in the constant cost of pairwise communications scales as $O\left(\left(m.\frac{K}{N}\right)^4\right)$. \\
 
2. {\bf Single Agent running \name} - In case of no communication, when a single agent runs Algorithm \ref{algo:main-algo} (without requiring communication graph $G$), it incurs a higher regret due to subspace exploration (which scales as $O(Km\sqrt{T})$ instead of $O((K/N)m\sqrt{T})$ in the multi-agent case), because it has to search through all the $K$ subspaces to find the subspace containing $\theta^{*}$. We express this result formally in Theorem \ref{thm:mainresult_single}, which is given in Appendix \ref{proof:main_single}.\\

3. Setting $\delta = \frac{1}{T}$ in Theorem \ref{thm:mainresult} requires the knowledge of time horizon in \name \ Algorithm to achieve the corresponding regret guarantee. However, this is not a problem, as a fixed value of the confidence parameter $\delta \in (0, 1)$ achieves the same regret scaling as in Theorem \ref{thm:mainresult} with high probability, which can be proved in a similar manner.
Thus, the insights that can be obtained from our results are unaffected by the knowledge of time horizon.\\

4. {\bf Subspace recommendation quality vs. network spread} - Observe that $b > 1$ is an input to the algorithm, where agents communicate for the $l^{th}$ time after playing $\lceil b^{l-1} \rceil$ number of times since the last communication. Thus, increasing $b$ will decrease the total number of communications between agents. Theorem \ref{thm:mainresult} shows that, there exists an optimal $b^* > 1$, such that $b^* = \arg \min_{b>1} \mathbb{E}[R_{T}^{(i)}]$. This can be seen by observing that as $b$ decreases towards $1$, the time between two communication instants reduces. However, each communication is based on fewer samples and thus, subspace recommendations are noisy. On the other hand, as $b$ becomes large, each recommendation is based on large number of samples and thus, less noisy. The number of communications, however, is much lower, leading to a large time for the best subspace to spread. The optimal $b^{*}$ trades-off between these two competing effects.

\subsection{Impact of Network Structure on Regret}
We can obtain the dependence of regret bound on network-related parameters by expressing the term $\mathbb{E}[b^{2\tau_{\mathrm{spr}}^{(G)}}]$ in terms of the conductance $\phi$ of the gossip matrix (graph) $G$. In order to do so, we use a result obtained in \cite{aistats_gossip}, Corollary 17, which we reproduce here:

For a $d$-regular\footnote{standard graph-theoretic notion which has nothing to do with ambient dimension $d$ in $\mathbb{R}^d$} graph with adjacency matrix $A_{G}$, conductance $\phi$ and gossip matrix $G=d^{-1}A_{G}$, $\mathbb{E}[b^{2\tau_{\mathrm{spr}}^{(G)}}] \leq b^{\frac{2C \log N}{\phi}}$ for all $b \leq \mathrm{exp}\left(\frac{\phi}{C}\right)$, where $C$ is a universal constant.

Using the above result, we now consider an illustrative example 
in which we assume that the agents are connected by a complete graph, i.e., $G(i,j)=\frac{1}{N-1}$ for $j \neq i$, $0$ otherwise.
In this case, it is easy to see that for all $N$ and $b \leq \exp\left(\frac{N}{2(N-1)C}\right)$ (where $C$ is an universal constant), $\mathbb{E}[b^{2\tau_{\mathrm{spr}}^{(G)}}] \leq \alpha N^{2(\log_{2}b+ \log b)}$ for some constant $\alpha>0$, independent of $N$ (see Corollary 16 of \cite{aistats_gossip}, where we substitute the conductance $\phi = \frac{N}{2(N-1)}$ for the complete graph). This is because for the complete graph, $\tau_{\mathrm{spr}}^{(G)} \leq \log_{2}N + \log N$ with high probability. Corollary \ref{cor_regret} quantifies the impact of underlying network on regret scaling.
\begin{corollary}
\label{cor_regret}
Suppose the agents are connected by a complete graph and 

\noindent$b=\min\left\{\exp\left(\frac{\log 2}{1+\log 2}.\frac{1}{2}\right), \exp\left(\frac{N}{2(N-1)C}\right)\right\}$, i.e., $\log_{2}b+ \log b \leq \frac{1}{2}$. 
With the same assumptions for $\lambda$ and $\delta$ as in Theorem \ref{thm:mainresult}, the regret scaling of any agent $i \in [N]$ after playing \name \ Algorithm for $T$ time steps is given by
\begin{multline}
  \label{eqn:cor_regret}
    \mathbb{E}[R_T^{(i)}] \leq \underbrace{O(m\sqrt{T} \log T)}_{\text{Projected LinUCB Regret}} + \underbrace{O \left(\frac{K}{N}m \sqrt{T} \right)}_{\text{Cost of Subspace Exploration}} + \underbrace{O(N)}_{\mathbb{E}[b^{2\tau_{\mathrm{spr}}^{(G)}}]} +  O\left(\left(m.\frac{K}{N}\right)^4\right) + O\left(\frac{m^{4}N}{\Delta^6}\right),
\end{multline}
where $C > 0$ is an universal constant. 
\end{corollary}

In (\ref{eqn:cor_regret}), the $O(\cdot)$ notation only hides input constants and universal constants. 
It is evident from (\ref{eqn:cor_regret}) that the network structure does not affect the time scaling in regret. 

\subsection{Proof Sketch}
\label{proofsketch}
We provide the proof of Theorem \ref{thm:mainresult} in Appendix \ref{proof:main}, however, we summarize its salient ideas here. Similar to the phenomenon in unstructured bandits \cite{aistats_gossip}, we prove in Proposition \ref{prop:strong_freeze} that in our linear bandit setup, after a \emph{freezing phase} $\tau$, all agents have the correct subspace containing $\theta^{*}$. Consequently, for all phases $j > \tau$, all agents will play Projected LinUCB \cite{hassibi_PSLB} from the correct subspace in the exploit time instants and recommend it at the end of the phase. Therefore, the set of subspaces every agent has does not change after phase $\tau$ and the regret after phase $\tau$ can be decomposed into regret due to pure exploration and regret due to Projected LinUCB (Proposition \ref{prop:regdecomp}).

The technical novelty of our proof lies in bounding the regret till phase $\tau$, i.e., $\mathbb{E}\left[\tau + \frac{b^{\tau}-1}{b-1}\right]$ (Proposition \ref{prop:reg_btau}) and in particular showing it to be finite. This follows from two key observations arising from pure exploration in the explore time steps. First, for any agent $i \in [N]$ and subspace ID $k \in S_{j}^{(i)}$, the estimate $\widetilde{\theta}_{k,j}^{(i)}$ of $\theta^*$ using only its explore samples up to and including phase $j$, concentrates to $P_{k}\theta^{*}$ in $l_2$ norm (Lemma \ref{lemma:intstep}). Therefore, for any subspace ID $k$, $\|\widetilde{\theta}_{k,j}^{(i)}\|_2$ concentrates to $\|P_{k}\theta^*\|_2$ and eventually, the correct subspace $\mathrm{span}(U_1)$ will achieve the largest value of  $\|\widetilde{\theta}_{k,j}^{(i)}\|_2$ among all subspaces, if present in the active set. Subsequently, we use the previous fact in Lemma \ref{lemma:wrongsubspace} to show that if an agent has the correct subspace containing $\theta^{*}$ in a phase, the probability that it will not be recommended and hence dropped at the end of the phase is small.

Combining these two observations we establish that, after a random phase denoted by $\widehat{\tau}_{\mathrm{stab}}$ in Appendix \ref{proof:main}, satisfying $\mathbb{E}[\widehat{\tau}_{\mathrm{stab}}]<\infty$, agents never recommend incorrectly at the end of a phase and thus play the Projected LinUCB on the correct subspace in the exploit time instants of a phase. 
To conclude, after random phase $\widehat{\tau}_{\mathrm{stab}}$, the spreading time can be coupled with that of a standard rumor spreading \cite{rumor_spread}, as once an agent learns the correct subspace, it is not dropped by the agent.
This final part is similar to the one conducted for unstructured bandits in \cite{aistats_gossip}, 
giving us the desired bound on $\mathbb{E}\left[\tau + \frac{b^{\tau}-1}{b-1}\right]$. \\


\noindent {\bf Remark (Freezing time):} The freezing phase $\tau$ is a quantity only showing up in the analysis, but is not part of the algorithm. In fact, the algorithm needs all agents to explore and communicate in all phases indefinitely, because $\tau$ is a sample-path dependent quantity. Indeed, any bandit algorithm that can achieve sub-linear cumulative regret in the stochastic setting inherently has such a freezing time with finite expectation (including in the classical single-agent $K$-armed bandit); beyond this time, the best arm is identified with high probability. This can be shown by noting that sub-linear regret implies that the probability the best arm is not played at time $t$, decreases to $0$ as $t$ goes to infinity. The finite freezing time follows from the simple Hoeffding inequality and Borel-Cantelli lemma. However, this is not useful in the algorithm but serves only as a proof technique. Formally, this random time $\tau$ is not a stopping time, and cannot be determined in an online fashion. Moreover, despite the existence of such a freezing time, the lower bounds for regret increase with the time horizon, showing that infinite exploration is necessary.
\\

\noindent {\bf Remark (Technical differences w.r.t. \cite{aistats_gossip})}:
The algorithms in \cite{aistats_gossip} and the \name \ Algorithm appear similar, because of the correspondence between the subspaces in our setup and the arms in a $K$-armed bandit. 
However, this correspondence is superficial, because unlike an arm, a subspace represents a continuum of actions, instead of just being an action.
In order to quantify the reward corresponding to a subspace, one has to form an estimate of $\theta^*$ in that subspace by playing the sequence of action vectors spanning that subspace.

Furthermore, any given phase in \name \ bears a superficial resemblance to the \emph{explore-then-commit} (ETC) algorithm, wherein the explore part of the phase that identifies the subspace to commit to in the exploit part, analogous to best arm identification in the standard $K$-armed bandit.
However, the ETC algorithm requires the knowledge of lower bound of arm mean gaps as an input. In our model, this translates to agents needing knowledge of the distance between subspaces (denoted by $\Delta_k$ for all $k \neq 1$), which requires knowledge of $\theta^*$.
We circumvented this through a phased approach with exponentially increasing lengths, where each phase has an explore part and a commit part. 
The phases with exponentially increasing lengths ensure that: (i) the probability of picking a subspace not containing $\theta^*$ decreases with every phase, (ii) the increasing duration of playing Projected LinUCB within a phase as the phases progress minimizes the cumulative regret, and (iii) does not need knowledge of gap between subspaces. 
The consequence of not knowing $\theta^*$ is why agents need to continually explore in the explore part of every phase, as opposed to only exploring once in the beginning.
\\


\noindent {\bf Remark (Discussion on a lower bound)}:
We provide a brief discussion about the fundamental limits of our model (in terms of cumulative regret) to evaluate the effectiveness of \name \ Algorithm.
We conjecture that the regret of $\Omega(m\sqrt{T})$ is unavoidable.

The above claim can be argued as follows: in our model, an agent can exchange at most $\log_2 K + 1$ number of bits each time it chooses to communicate. 
One can consider the scenario in which whenever an agent decides to pull subspace recommendation (as a subspace ID in $\{1, \cdots, K\}$ can be perfectly described by $\log_2 K + 1$ number of bits) from another agent based on gossip matrix $G$, suppose it always receives the ID of the correct subspace containing $\theta^*$. 
In that case, the agent doesn't have to incur any regret in finding the correct subspace. 
However, given that there is no sample (action vectors and corresponding rewards) sharing possible between the agents, an agent will still have to search for $\theta^*$ in the correct subspace by itself. 
From \cite{lattimore_book}, Chapter 24, we know that finding $\theta^*$ in $\mathbb{R}^{d}$ without any side information results in $\Omega(d\sqrt{T})$ regret.
Given that the subspaces are $m$-dimensional, we can replace $d$ with $m$ in the previous statement and conclude that finding $\theta^*$ in the correct subspace incurs a regret of $\Omega(m \sqrt{T})$.

However, formalizing this argument requires surmounting some technical challenges. First, we need to precisely define the space of allowed communication policies without prescribing the content of the messages, for example, those that communicate at most a fixed number of bits at each time instant and total number of bits that scales as the logarithm of the time horizon. Once this is done, we need to establish that no communication policy under this constraint can encode knowledge of the true underlying $\theta^{*}$ to small enough precision, and show that communicating information other than subspace indices does not yield regret reduction. While the preceding paragraph provides a plausible intuition for the $\Omega(m\sqrt{T})$ lower bound, a detailed argument is left to future work.  



\section{Benefit of Collaboration in High Dimensions}
\label{collab_benefit}
In this section, we illustrate how collaboration aids in reducing regret for each agent in the high-dimensional setting. We quantify this by computing the ratio of the regret upper bound achieved by \name \ Algorithm without collaboration to that achieved with collaboration for any agent, denoted by $r_{C}(T)$.
The \emph{high-dimensional} setting corresponds to large $d$, $m$ a constant, $K$ and $N$ scaling linear in $d$ (system with a large number of agents).

\begin{corollary}
\label{cor:collabbenefit}
Consider a high-dimensional system where $N$ agents are connected by a complete graph with $K = N = \frac{d}{m}$, where $d$ is a multiple of $m$ and $m$ is a constant. Assume that $d \geq 3m$. With the choice of $b$ as in Corollary \ref{cor_regret},  

(a) for any agent $i \in [N]$, the regret with collaboration scales as $O(m\sqrt{T} \log T)$,

(b) $r_{C}(T) = \frac{r_{S}(T)}{r_{M}(T)}$, where
	\begin{multline*}
	r_{S}(T) = \sqrt{8mT\beta_{T}^{2}\log\left(1+\frac{T}{m\lambda}\right)}+2+ 2g(b)\left(\lceil b(16d)^2 \rceil + \frac{8b^2}{\log b}.\frac{m^2}{\Delta^2}\right)\\+ 16d \log_{b}(h_{b, T}) + 16 d \frac{\sqrt{h_{b, T}}-1}{\sqrt{b}-1},
	\end{multline*}
	\begin{multline*}
	r_{M}(T) = \sqrt{8mT\beta_{T}^{2}\log\left(1+\frac{T}{m\lambda}\right)}+2+ 2g(b)\left(\lceil b^{2} (48m)^4 \rceil + \frac{48b^3}{\log b}.\frac{m^{3}d}{\Delta^6}+\frac{\alpha d}{m}\right) \\+ 48 m \log_{b}(h_{b, T})
+ 48 m \frac{\sqrt{h_{b, T}}-1}{\sqrt{b}-1}.
	\end{multline*}
\end{corollary}

\begin{proof}
The proof follows by substituting $K=N=\frac{d}{m}$ in Theorems \ref{thm:mainresult} and \ref{thm:mainresult_single}, along with the bound for spreading time $\mathbb{E}[b^{2\tau_{\mathrm{spr}}^{(G)}}]$ from Corollary \ref{cor_regret}. 
\end{proof}

The following {\bf observations} can be deduced from point $(b)$ of Corollary \ref{cor:collabbenefit}:


(i) when $T = \Theta\left(d^{1+\gamma}\right)$ for all $\gamma \geq 0$, $r_{C}(T) = \Omega\left(\frac{d}{\log d}\right)$.

(ii) when $T = \Theta(e^{d^\beta})$ for all $\beta \in (0, 1)$, $r_{C}(T) = \Omega(d^{1-\beta})$.

(iii) when $T = \Omega(e^{d})$, $r_{C}(T) = \Omega(1)$.
\\

Two remarks are now in order.
\\

{\bf{Remarks}:}

1. \textbf{Matching an oracle's regret rate, asymptotically:} \ref{cor:collabbenefit} shows the power of collaboration in a large multi-agent system, as the regret scaling for any agent $i \in [N]$ matches that of a genie who is already aware of the subspace containing $\theta^*$ and can play Projected LinUCB on that subspace \cite{hassibi_PSLB}. This demonstrates that the cost of subspace search can be amortized across agents and only contributes a lower order term in regret, despite agents communicating infrequently (a total of $O(\log T)$ number of pairwise communications by every agent) and exchanging a limited number of bits in each communication (no sample sharing). 
Furthermore, the discussion in Section \ref{subgossreg} implies that in the absence of sample sharing between agents, an agent will incur $\Omega(m\sqrt{T})$ regret for finding $\theta^*$ in the correct subspace.
Thus, \name \ Algorithm is \emph{near-optimal} even in high-dimensional settings with large number of subspaces and agents. 

2. \textbf{Finite-time gains due to faster search of subspaces with collaboration:} The observations following the Corollary \ref{cor:collabbenefit} show that even the single agent running \name \ without communications is able to utilize the side information and incur lower regret ($O(m\sqrt{T}\log T + d\sqrt{T})$), compared to an agent running OFUL \cite{oful} without any side information.
However, the time taken by the single agent running \name \ to reap the benefits of the subspace side information is very large in high-dimensional settings ($T=\Omega(e^d)$).
In contrast, the ability of a multi-agent system to learn the right subspace faster is what leads to large collaborative gain of $r_{C}(T)= \Omega\left(\frac{d}{\log d}\right)$ by time $T=\Omega(d)$. 
These gains are also observed empirically in Figure \ref{fig1}.
These gains are more pronounced and are observed for large duration of time in settings with large $d$, which is typical in many modern applications.

\section{Numerical Results}

We evaluate the \name \ algorithm empirically in synthetic simulations. We show the cumulative regret (after averaging across all agents) over $30$ random runs of the algorithm with $95\%$ confidence intervals. We compare its performance with two benchmarks: \name \ algorithm with no collaborations (i.e., a single agent playing \name \ algorithm) and a single agent playing the OFUL (classical LinUCB) algorithm of \cite{oful}. In this section, the number of times a subspace $\mathrm{span}(U_k)$ (where $k \in S_{j}^{(i)}$) is explored during the {\sc{Explore}} subroutine in phase $j$ is set to $m \lceil b^{\frac{j-2}{2}} \rceil$, as the constants in \name \ algorithm (Algorithm \ref{algo:main-algo}) arise from somewhat loose tail bounds.

\begin{figure*}[h]
\centering
\includegraphics[width=\textwidth]{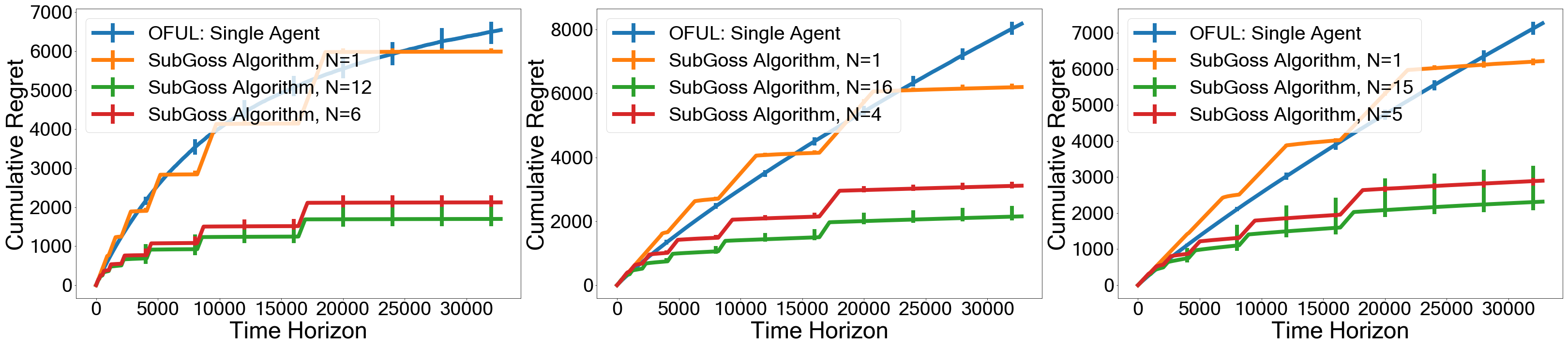}
\caption{Illustrating benefit of collaboration. $(d, m, K)$ are (24, 2, 12), (48, 3, 16), and (60, 4, 15) respectively.}
\label{fig1}
\end{figure*}

In our experiments, the agents are connected through a complete graph. 
Each $m$-dimensional subspace is the orthogonal matrix obtained by the SVD of a random $d \times m$ matrix with i.i.d. standard normal entries.
The action set $\mathcal{A}$ consists of $5d$ i.i.d. Gaussian vectors on surface of the unit $l_2$ ball, along with orthonormal basis vectors for each of the $K$ subspaces.
The vector $\theta^*$ is the projected version of a standard Gaussian vector onto subspace $1$ (the true subspace).
We set $b=2$ and $\lambda = 1$ in simulations. 
Fig. \ref{fig1} evaluates the performance of \name \ algorithm for different values of problem parameters $(d,m,K,N)$.

{\bf Insights from numerical results:} 
From simulations we confirm several insights predicted by our theory. First, we see that \name\ yields lower regret than OFUL for the single agent case, demonstrating that \name\ can effectively leverage the side-information provided through the subspaces. Second, we observe the collaboration gains, where any agent in the multi-agent setting incurs far smaller regret compared to a single agent without collaboration. Finally, we also observe that as the number of agents increases, the regret for every agent decreases. These collaborative gains follow, as each agent has to search through a smaller set of subspaces to find the true subspace.

%


\section{Conclusions and Open Problems}
We studied a multi-agent linear bandit problem with side information (in the form of disjoint $m$-dimensional subspaces), where only one of the subspaces contains the unknown parameter $\theta^* \in \mathbb{R}^{d}$, but agents are unaware of the subspace containing it. 
We proposed a novel decentralized algorithm, where agents collaborate by sending recommendations through pairwise gossip communications across a communication graph connecting them, to minimize their individual cumulative regret. 
We demonstrated that distributing the search for the subspace containing $\theta^*$ across the agents and learning of the unknown vector in the corresponding low-dimensional subspace results in a much smaller per-agent regret, compared to the case when agents do not communicate. 
However, the paper leaves open, some important questions. The paper assumed that all agents have exact knowledge of the subspaces. In several practical applications however, the subspaces are estimated from historical data and as such can only be known noisily at best. Developing algorithms that can leverage benefit from collaboration while being robust to mis-specifications is an interesting direction for future work.  Another open problem is to establish lower bounds on regret under our model of information sharing. This is non-trivial to define since the communication budget needs to be accounted for in regret. To the best of our knowledge, lower bounds involving both communication and regret minimization have not been established even for the simple unstructured bandits case. 

\section*{Acknowledgements}
This work was partially supported by ONR Grant N00014-19-1-2566, NSF Grant SATC 1704778, ARO grant W911NF-17-1-0359, the NSA SoS Lablet H98230-18-D-0007 and the WNCG Industrial Affiliates Program.

\bibliographystyle{plain}
\bibliography{ref-lin-bandits}

\clearpage
\begin{appendices}
\section{Technical Assumption for Theorem \ref{thm:mainresult}} 
\label{sec:assumption_discussion}
Building on the communication constraints considered in \cite{aistats_gossip}, we make the following mild assumption:
\cite{aistats_gossip} The gossip matrix $G$ is \emph{irreducible}, i.e., for any $i \neq j \in [N]$, there exists $2 \leq l \leq N$ and $k_1,\cdots,k_l \in [N]$, with $k_1 = i$ and $k_l=j$ such that the product $G(k_1,k_2) \cdots G(k_{l-1},k_l) > 0$. In words, the communication graph among the agents is connected \cite{aistats_gossip}.

This assumption is needed because if the communication graph among the agents is not connected, the setup becomes \emph{degenerate}, as there exists at least a pair of agents which cannot participate in information exchange.
However, the practical insights that can be obtained from our results are not affected by this assumption. 

\section{Proof of Theorem \ref{thm:mainresult}}
\label{proof:main}
In this section and subsequent sections, we assume agents know the parameter $S$ such that $\|\theta^{*}\|_{2} \leq S$. In the paper, we set $S=1$ for ease of exposition. Before going through the proof, we will first set some definitions and notations.

\subsection{Definitions and Notations}
We adapt the proof ideas developed in \cite{aistats_gossip} for the unstructured bandit case.
Recall that for any phase $j$, agent $i \in [N]$, and subspace $k \in S_{j}^{(i)}$, $\widetilde{n}_{k,j}^{(i)}$ is the number of times agent $i$ explores the subspace $\mathrm{span}(U_k)$ up to and including phase $j$ and $\widehat{\mathcal{O}}_{j}^{(i)} = \arg \max_{k \in S_{j}^{(i)}} \|\widetilde{\theta}_{k,j}^{(i)}\|_{2}$.
In words, $\widehat{\mathcal{O}}_{j}^{(i)}$ is the ID of the subspace in which every agent $i \in [N]$ plays Projected LinUCB in the exploit time slots of phase $j$ and subsequently, recommends it at the end of phase $j$.
Let $\chi_{j}^{(i)}$ be the indicator variable for the event $\left\{1 \in S_{j}^{(i)}, \widehat{\mathcal{O}}_{j}^{(i)} \neq 1\right\}$, i.e., 
\begin{align*}
\chi_{j}^{(i)} = \mathbf{1}\left(1 \in S_{j}^{(i)}, \widehat{\mathcal{O}}_{j}^{(i)} \neq 1\right),    
\end{align*} 
which indicates whether agent $i$, \emph{if} it has the subspace $\mathrm{span}(U_1)$, does not recommend it at the end of a phase. Similar to \cite{aistats_gossip}, we provide the definitions of certain random times that will aid in the analysis:
\begin{align*}
\widehat{\tau}_{\mathrm{stab}}^{(i)} &= \inf \{ j^{'} \geq \tau_0 : \forall j \geq j^{'}, \chi_j^{(i)} = 0 \}, \\
\widehat{\tau}_{\mathrm{stab}} &= \max_{i \in [N]} \widehat{\tau}_{\mathrm{stab}}^{(i)},\\
\widehat{\tau}_{\mathrm{spr}}^{(i)} &= \inf\{j \geq \widehat{\tau}_{\mathrm{stab}} : 1 \in S_j^{(i)} \} - \widehat{\tau}_{\mathrm{stab}}, \\
\widehat{\tau}_{\mathrm{spr}} &= \max_{i \in \{1,\cdots,N\}} \widehat{\tau}_{\mathrm{spr}}^{(i)}, \\
{\tau} &= \widehat{\tau}_{\mathrm{stab}}+\widehat{\tau}_{\mathrm{spr}}.
\end{align*}
Here, $\widehat{\tau}_{\mathrm{stab}}^{(i)}$ is the earliest phase, such that \emph{if} agent $i$ has the subspace $\mathrm{span}(U_1)$ in the phases following it, it will recommend the subspace $\mathrm{span}(U_1)$. The number of phases it takes after $\widehat{\tau}_{\mathrm{stab}}$ to have the subspace $\mathrm{span}(U_1)$ in its playing set is denoted by $\widehat{\tau}_{\mathrm{spr}}^{(i)}$. The following proposition shows that the system is frozen after phase $\tau$, i.e. after phase $\tau$, the set of subspaces of all agents remain fixed in the future.

\begin{prop}
\label{prop:strong_freeze}
	 For all agents $i \in \{1,\cdots,N\}$, we have almost-surely,
	\begin{align*}
	\bigcap_{j \geq{\tau}} S_j^{(i)} &= S_{{\tau}}^{(i)}, \\
	\widehat{\mathcal{O}}_l^{(i)} &= 1 \text{ }\forall l \geq {\tau}, \text{ } \forall i \in \{1,\cdots,N\}.
	\end{align*}
\end{prop}

\begin{proof}
	For any agent $i \in [N]$ and any phase $j \geq \tau$, we have for all $j \geq \tau$, 
	\begin{equation}
	\chi_j^{(i)} = 0,
	\label{eqn:freeze_proof1}
	\end{equation}
    as $\tau \geq \widehat{\tau}_{\mathrm{stab}}^{(i)}$. However, as $\tau \geq \widehat{\tau}_{\mathrm{stab}} + \widehat{\tau}_{\mathrm{spr}}^{(i)}$, we know that
	\begin{equation}
	1 \in S_j^{(i)}.
	\label{eqn:freeze_proof2}
	\end{equation}
	Equations (\ref{eqn:freeze_proof1}) and (\ref{eqn:freeze_proof2}) imply that $\widehat{\mathcal{O}}_j^{(i)} = 1$. Moreover, $\widehat{\mathcal{O}}_j^{(i)} = 1$ is true for all phases $j \geq \tau$ and all agents $i \in [N]$, as they are arbitrarily chosen. Furthermore, the update step of the algorithm along with the above reasoning tells us that none of the agents will change their subspaces after any phase $j \geq \tau$, as the agents already have the correct subspace in their respective playing sets. Thus, $\bigcap_{j \geq{\tau}} S_j^{(i)} = S_{{\tau}}^{(i)}$ for all agents $i \in [N]$.
\end{proof}
Proposition \ref{prop:strong_freeze} also tells us that for all phases $j \geq \tau$, in the exploit time slots, all agents will play Projected LinUCB from the subspace $\mathrm{span}(U_1)$, because the algorithm picks the subspace $\mathrm{span}(U_{\widehat{\mathcal{O}}_{j}^{(i)}})$ in the exploit time slots of phase $j$ and $\widehat{\mathcal{O}}_{j}^{(i)} = 1$ for all $j \geq \tau$. 

\subsection{Intermediate Results}
Before stating and proving the intermediate results, we highlight the key pieces needed to prove Theorem \ref{thm:mainresult}. We already showed in Proposition \ref{prop:strong_freeze} that after the freezing phase $\tau$, all agents have the correct subspace containing $\theta^{*}$ and recommend it henceforth. Thus, the expected cumulative regret incurred can be decomposed into two parts: the regret up to phase $\tau$ and the regret after phase $\tau$. 

The expected cumulative regret incurred up to phase $\tau$ is a constant independent of the time horizon (Proposition \ref{prop:reg_btau}). It is a consequence of following important observations resulting from pure exploration in the explore time steps:
\begin{itemize}[leftmargin=*]
\item For any agent $i \in [N]$ and subspace $k \in S_{j}^{(i)}$, the estimate $\widetilde{\theta}_{k,j}^{(i)}$ concentrates to $P_k \theta^{*}$ in $l_2$ norm (Lemma \ref{lemma:intstep}).
\item Subsequently, we show that the probability that an agent will not recommend and thus drop the correct subspace containing $\theta^{*}$ is small at the end of a phase (Lemma \ref{lemma:wrongsubspace}).
\end{itemize}
The above observations imply that after a (random) phase, denoted by $\widehat{\tau}_{\mathrm{stab}} \leq \tau$, agents always recommend (and never drop) the correct subspace. After phase $\widehat{\tau}_{\mathrm{stab}}$, we stochastically dominate (in Proposition \ref{prop:couple_spread}) the spreading time of the correct subspace with a standard rumor spreading process \cite{rumor_spread}. Hence, the expected cumulative regret up to phase $\tau$ is bounded by the total number of time steps taken to reach phase $\widehat{\tau}_{\mathrm{stab}}$ and the additional number of phases taken to spread the correct subspace.


Post phase $\tau$, as the active set of subspaces maintained by agents remains unchanged (as deduced in Proposition \ref{prop:strong_freeze}) and thus, the regret can be decomposed into sum of regret due to pure exploration and regret due to projected LinUCB. The regret due to projected LinUCB is adapted from the analysis of a similar algorithm conducted in \cite{hassibi_PSLB}. 

The following intermediate results will precisely characterize the intuition behind the proof of Theorem \ref{thm:mainresult}.

\begin{prop}
\label{prop:regdecomp}
The regret of any agent $i \in \{1,\cdots,N\}$ after playing for $T$ steps is bounded by
\begin{multline*}
\mathbb{E}[R_{T}^{(i)}] \leq 2S\left(\mathbb{E}\left[\tau + \frac{b^{\tau}-1}{b-1}\right]\right) + \mathbb{E}[R_{\mathrm{proj},T}]+16 mS \left(\frac{K}{N}+2\right) \log_{b}(h_{b, T}) \\
+ 16 mS \left(\frac{K}{N}+2\right) \frac{\sqrt{h_{b, T}}-1}{\sqrt{b}-1}, 
\end{multline*}
where $h_{b, T}$ is defined in Theorem \ref{thm:mainresult}.
\end{prop}
\begin{proof}
We will first show that the instantaneous regret $w_{t}^{(i)} \leq 2S$ for all $i \in [N]$. In order to obtain this bound, notice that for any $a \in \mathbb{R}^{d}$ such that $\|a\|_2 \leq 1$,
\begin{align*}
    |\langle \theta^{*}, a \rangle| \leq ||\theta^{*}||_{2}.||a||_{2} \leq S
\end{align*}
by Cauchy-Schwarz inequality. Therefore, we have $w_{t}^{(i)} \leq 2S$ for all $t$. 

Let $l \in \mathbb{N}$ such that \name \ Algorithm is played for $t$ steps by the end of phase $l$. $t$ and $l$ are related as follows:
\begin{align}
    t= \sum_{p=1}^{l} \lceil b^{p-1} \rceil. \label{eq:phasetimecorr}
\end{align}
Therefore, $\frac{b^l - 1}{b-1} \leq t \leq l + \frac{b^l - 1}{b-1}$. Assume that \name \ Algorithm is played for $T$ steps such that $T$ occurs in some phase $E$, i.e., $\frac{b^{E-1} - 1}{b-1} + 1 \leq T \leq E + \frac{b^E - 1}{b-1}$ and it follows that $E \leq \log_{b}(b(1+(T-1)(b-1))) = \log_{b}(h_{b, T})$.

Let $e_{j} = \sum_{l=1}^{j} \lceil b^{l-1} \rceil$ denote the number of times \name \ Algorithm has been played by the end of phase $j$ and $\mathrm{Reg}_{j}^{(i)}$ denote the regret incurred by agent $i$ in phase $j$, i.e., $\mathrm{Reg}_{j}^{(i)} = \sum_{s=1}^{\lceil b^{j-1} \rceil} w_{e_{j-1}+s}$. From the definition of regret $R_{T}^{(i)}$,
\begin{align}
R_{T}^{(i)} &= \sum_{t=1}^{T} w_{t}^{(i)}\nonumber\\
&\leq \sum_{j=1}^{E} \mathrm{Reg}_{j}^{(i)}\nonumber\\
&= \sum_{j=1}^{\tau} \mathrm{Reg}_{j}^{(i)} + \sum_{j=\tau+1}^{E} \mathrm{Reg}_{j}^{(i)}\label{eq:regdecompint}.
\end{align}
We will now bound each of the terms in (\ref{eq:regdecompint}) separately. The first term $\sum_{j=1}^{\tau} \mathrm{Reg}_{j}^{(i)}$ can be bounded as follows:
\begin{align}
    \sum_{j=1}^{\tau} \mathrm{Reg}_{j}^{(i)} = \sum_{j=1}^{\tau} \sum_{s=1}^{\lceil b^{j-1} \rceil} w_{e_{j-1}+s}^{(i)} 
    \leq 2S \sum_{j=1}^{\tau} \sum_{s=1}^{\lceil b^{j-1} \rceil} 1 
    = 2S \sum_{j=1}^{\tau} \lceil b^{j-1} \rceil
    \leq 2S \left(\tau + \frac{b^{\tau}-1}{b-1}\right), \label{eq:regfreeze}
\end{align}
where the second step follows from $w_{t} \leq 2S$ for all $t \in \mathbb{N}$ and the last step follows from the fact that $\lceil x \rceil \leq x+1$ for all $x \in \mathbb{R}$.
We bound the second term $\sum_{j=\tau+1}^{E} \mathrm{Reg}_{j}^{(i)}$ in the following steps: let $d_{j}^{(i)} = 8 m |S_{\tau}^{(i)}| \lceil b^{\frac{j-1}{2}}\rceil$ and $R_{\mathrm{proj},T}$ denote the regret incurred by playing Projected LinUCB on the subspace containing $\theta^*$ after the freezing phase $\tau$, i.e., $R_{\mathrm{proj},T} = \sum_{j=\tau+1}^{E}\sum_{d_{j}^{(i)}+1}^{\lceil b^{j-1} \rceil} w_{e_{j-1}+s}^{(i)}$. Then, we have
\begin{align}
    \sum_{j=\tau+1}^{E} \mathrm{Reg}_{j}^{(i)} &= \sum_{j=\tau+1}^{E} \sum_{s=1}^{\lceil b^{j-1} \rceil} w_{e_{j-1}+s}^{(i)}\nonumber\\
    &{=} \sum_{j=\tau+1}^{E} \sum_{s=1}^{d_{j}^{(i)}} w_{e_{j-1}+s}^{(i)} + \sum_{j=\tau+1}^{E} \sum_{d_{j}^{(i)}+1}^{\lceil b^{j-1} \rceil} w_{e_{j-1}+s}^{(i)}\nonumber\\
    &\stackrel{(a)}{\leq} 2S \sum_{j=1}^{E}\sum_{s=1}^{8 m |S_{\tau}^{(i)}| \lceil b^{\frac{j-1}{2}}\rceil} 1 + R_{\mathrm{proj},T} \nonumber\\
    &\stackrel{(b)}{\leq} 16 mS \left(\frac{K}{N}+2\right) \log_{b}(h_{b, T}) + 16 mS \left(\frac{K}{N}+2\right) \frac{\sqrt{h_{b, T}}-1}{\sqrt{b}-1} + R_{\mathrm{proj},T}. \label{eq:regdecompaftertau} 
\end{align}

Recall that for any agent $i$, \name \ Algorithm explores in the first $d_{j}^{(i)} = 8m |S_{\tau}^{(i)}| \lceil b^{\frac{j-1}{2}}\rceil$ time slots of phase $j$ by playing the orthonormal basis vectors of each of the subspaces in the playing set in a round robin fashion. Therefore, in step $(a)$, we bound the total number of explore steps from phase $j > \tau$ (first term) by the bound on total number of explore steps from $t=1$ to $T$. In the remaining time slots of phases $j \in \mathbb{N}$, agents play Projected LinUCB in the subspace $\mathrm{span}(U_{\widehat{\mathcal{O}}_{j}^{(i)}})$ and $\widehat{\mathcal{O}}_{j}^{(i)} = 1$ for all $j > \tau$. Thus, the second term in step $(a)$ is bounded above by the regret incurred by playing Projected LinUCB in the subspace $\mathrm{span}(U_{1})$ for $T$ time steps (as the number of times an agent will play Projected LinUCB is less than $T$). Step $(b)$ follows from the discussion that if time step $T$ occurs in some phase $E$ then $E \leq \log_{b}(h_{b, T})$, $|S_{j}^{(i)}| \leq \frac{K}{N}+2$ for all $j \in \mathbb{N}$, and $\lceil x \rceil \leq x+1$ for all $x \in \mathbb{R}$.

Substituting (\ref{eq:regfreeze}) and (\ref{eq:regdecompaftertau}) in (\ref{eq:regdecompint}), followed by taking expectation on both sides completes the proof of Proposition \ref{prop:regdecomp}.
\end{proof}

The following lemma bounds the probability that for any subspace $k \in S_{j}^{(i)}$, $\widetilde{\theta}_{k,j}^{(i)}$ deviates from $P_{k}\theta^{*}$ in $l_2$ norm after the explore time slots in phase $j$, which will eventually help us obtain a bound on probability of picking the wrong subspace. 

\begin{lemma}
\label{lemma:intstep} 
For any agent $i \in [N]$, phase $j \geq \tau_0$, and $k \in S_{j}^{(i)}$, we have
\begin{align*}
    \mathbb{P}\left(\|\widetilde{\theta}_{k,j}^{(i)} - P_{k}\theta^{*}||_2 > \epsilon\right) \leq 2m \ \exp\left(-\frac{4\epsilon^2}{m} b^{\frac{j-1}{2}}\right).
\end{align*}
where $\epsilon > 0$.
\end{lemma}

\begin{proof}
We have for any $k \in S_{j}^{(i)}$,
\begin{align*}
    \widetilde{\theta}_{k,j}^{(i)} &= \arg \min_{\theta \in \mathbb{R}^{d}} \left\|(\widetilde{A}_{k,\widetilde{n}_{k,j}^{(i)}}^{(i)})^{T} \theta - \widetilde{r}_{k,\widetilde{n}_{k,j}^{(i)}}^{(i)}\right\|_{2}\\
    &= \arg \min_{\theta \in \mathbb{R}^{d}} \left\|(P_{k}\widetilde{A}_{k,\widetilde{n}_{k,j}^{(i)}}^{(i)})^{T} \theta - \widetilde{r}_{k,\widetilde{n}_{k,j}^{(i)}}^{(i)}\right\|_{2},
\end{align*}
where the last step follows from the fact that during the explore time slots, orthonormal basis vectors for each of the subspaces in $S_{j}^{(i)}$ are played in a round robin fashion. By squaring the objective function in the last step, taking the gradient and setting it to all zeroes vector, we get
\begin{align}
    \widetilde{\theta}_{k,j}^{(i)} = (P_{k}\widetilde{A}_{k,\widetilde{n}_{k,j}^{(i)}}^{(i)}\widetilde{A}_{k,\widetilde{n}_{k,j}^{(i)}}^{(i)^T}P_{k})^{\dagger}(P_{k}\widetilde{A}_{k,\widetilde{n}_{k,j}^{(i)}}^{(i)}\widetilde{A}_{k,\widetilde{n}_{k,j}^{(i)}}^{(i)^T}P_{k})\theta^{*} + (P_{k}\widetilde{A}_{k,\widetilde{n}_{k,j}^{(i)}}^{(i)}\widetilde{A}_{k,\widetilde{n}_{k,j}^{(i)}}^{(i)^T}P_{k})^{\dagger} P_{k}\widetilde{A}_{k,\widetilde{n}_{k,j}^{(i)}}^{(i)} \eta_{k,\widetilde{n}_{k,j}^{(i)}}^{(i)}, \label{eq:thetaexplore}
\end{align}

where $M^{\dagger}$ denotes the Moore-Penrose pseudoinverse of the matrix $M$. By substituting $P_{k}=U_{k}U_{k}^{T}$, we get $P_{k}\widetilde{A}_{k,\widetilde{n}_{k,j}^{(i)}}^{(i)}\widetilde{A}_{k,\widetilde{n}_{k,j}^{(i)}}^{(i)^T}P_{k} = U_{k} \widetilde{\Sigma}_{k,\widetilde{n}_{k,j}^{(i)}}^{(i)}U_{k}^{T}$, 
where $\widetilde{\Sigma}_{k,\widetilde{n}_{k,j}^{(i)}}^{(i)} = (U_{k}^{T}\widetilde{A}_{k,\widetilde{n}_{k,j}^{(i)}}^{(i)})(U_{k}^{T}\widetilde{A}_{k,\widetilde{n}_{k,j}^{(i)}}^{(i)})^{T}$. Notice that $\widetilde{\Sigma}_{k,\widetilde{n}_{k,j}^{(i)}}^{(i)}$ is a symmetric, full-rank $m \times m$ matrix, as $\widetilde{A}_{k,\widetilde{n}_{k,j}^{(i)}}^{(i)}$ is a matrix whose columns are the orthonormal basis vectors of the subspace $\mathrm{span}(U_k)$ in a round robin fashion and $\widetilde{n}_{k,j}^{(i)} > m$. Therefore, $(P_{k}\widetilde{A}_{k,\widetilde{n}_{k,j}^{(i)}}^{(i)}\widetilde{A}_{k,\widetilde{n}_{k,j}^{(i)}}^{(i)^T}P_{k})^{\dagger} = U_{k} (\widetilde{\Sigma}_{k,\widetilde{n}_{k,j}^{(i)}}^{(i)})^{-1}U_{k}^{T}$ and thus,  

\noindent$(P_{k}\widetilde{A}_{k,\widetilde{n}_{k,j}^{(i)}}^{(i)}\widetilde{A}_{k,\widetilde{n}_{k,j}^{(i)}}^{(i)^T}P_{k})^{\dagger}(P_{k}\widetilde{A}_{k,\widetilde{n}_{k,j}^{(i)}}^{(i)}\widetilde{A}_{k,\widetilde{n}_{k,j}^{(i)}}^{(i)^T}P_{k}) = U_{k} U_{k}^{T} = P_{k}$. Moreover, as $\widetilde{A}_{k,\widetilde{n}_{k,j}^{(i)}}^{(i)} = [U_k \cdots U_k]_{d \times \widetilde{n}_{k,j}^{(i)}}$, $U_{k}^{T} \widetilde{A}_{k,\widetilde{n}_{k,j}^{(i)}}^{(i)} = [I_m \cdots I_m]_{m \times \widetilde{n}_{k,j}^{(i)}}$ (where $I_m$ denotes the $m \times m$ identity matrix) and thus, $\widetilde{\Sigma}_{k,\widetilde{n}_{k,j}^{(i)}}^{(i)} = \frac{\widetilde{n}_{k,j}^{(i)}}{m}I_m$. Substituting everything above in (\ref{eq:thetaexplore}), we get
\begin{align*}
    \widetilde{\theta}_{k,j}^{(i)} = P_{k}\theta^* + U_{k} v_{\widetilde{n}_{k,j}^{(i)}}^{(i)},
\end{align*}
where $v_{\widetilde{n}_{k,j}^{(i)}}^{(i)}$ is a $m \times 1$ vector whose entries are $v_{\widetilde{n}_{k,j}^{(i)},n}^{(i)} = \frac{m}{\widetilde{n}_{k,j}^{(i)}} \sum_{p=1: a_{k,p}^{(i)} = u_{k,n}}^{\widetilde{n}_{k,j}^{(i)}} \eta_{p}^{(i)}$ for all $n \in [m]$, $a_{k,p}$ denotes the $p^{\mathrm{th}}$ column of $\widetilde{A}_{k,\widetilde{n}_{k,j}^{(i)}}^{(i)}$, and $u_{k,n}$ denotes the $n^{\mathrm{th}}$ column of $U_k$. Hence, 
\begin{align*}
    \|\widetilde{\theta}_{k,j}^{(i)} - P_{k}\theta^*\|_{2}^{2} =  v_{\widetilde{n}_{k,j}^{(i)}}^{(i)^T} U_{k}^{T} U_{k} v_{\widetilde{n}_{k,j}^{(i)}}^{(i)} = \|v_{\widetilde{n}_{k,j}^{(i)}}^{(i)}\|_{2}^{2}, 
\end{align*}
where the above equality follows from the fact that $U_k$ is an orthonormal matrix. From the assumption that the additive noise is conditionally $1$-subgaussian, we know that 
\begin{align}
    \mathbb{P}\left(|v_{\widetilde{n}_{k,j}^{(i)},n}^{(i)}| > \gamma\right) \leq 2 e^{-\frac{\gamma^2 \widetilde{n}_{k,j}^{(i)}}{2m}}
    \label{singleidxerror}
\end{align}
for all $\gamma > 0$. If $|v_{\widetilde{n}_{k,j}^{(i)},n}^{(i)}| \leq \frac{\epsilon}{\sqrt{m}}$ for all $n \in [m]$ and $\epsilon > 0$, then $\|\widetilde{\theta}_{k,j}^{(i)} - P_{k}\theta^*\|_{2} \leq \epsilon$. Hence,
\begin{align*}
    \mathbb{P}\left(\|\widetilde{\theta}_{k,j}^{(i)} - P_{k}\theta^*\|_{2} > \epsilon\right) &\leq \mathbb{P}\left(\exists n \in [m]: |v_{\widetilde{n}_{k,j}^{(i)},n}^{(i)}| > \frac{\epsilon}{\sqrt{m}}\right)\\
    &=\mathbb{P}\left(\bigcup_{n=1}^{m}\left(|v_{\widetilde{n}_{k,j}^{(i)},n}^{(i)}| > \frac{\epsilon}{\sqrt{m}}\right)\right)\\
    &\stackrel{(a)}{\leq} \sum_{n=1}^{m} \mathbb{P} \left(|v_{\widetilde{n}_{k,j}^{(i)},n}^{(i)}| > \frac{\epsilon}{\sqrt{m}}\right)\\
    &\stackrel{(b)}{\leq} 2m \exp\left(-\frac{\epsilon^2 \widetilde{n}_{k,j}^{(i)}}{2m^2}\right)\\
    &\stackrel{(c)}{\leq} 2m \exp\left(-\frac{\epsilon^2}{2m^2}. 8m \lceil b^{\frac{j-1}{2}}\rceil\right)\\
    &\stackrel{(d)}{\leq} 2m \exp\left(-\frac{4\epsilon^2}{m} b^{\frac{j-1}{2}}\right).
\end{align*}
Step $(a)$ is a direct application of union bound. Step $(b)$ uses the result from (\ref{singleidxerror}). In step $(c)$, we use the fact that any subspace $k \in S_{j}^{(i)}$ is explored for at least $8m \lceil b^{\frac{j-1}{2}} \rceil$ times up to and including phase $j$. Step $(d)$ follows from the inequality $\lceil x \rceil \geq x$ for all $x \in \mathbb{R}$, thus concluding the proof.
\end{proof}

We will now obtain a bound on probability for choosing a wrong subspace. Since $\theta^{*} \in \mathrm{span}(U_1)$, any subspace chosen other than $\mathrm{span}(U_1)$ will result in an error. Mathematically, it can be expressed as $\widehat{\mathcal{O}}_{j}^{(i)} \neq 1$, which implies that at least one of the events below is true:
\begin{align}
    \left(\|\widetilde{\theta}_{\widehat{\mathcal{O}}_{j}^{(i)},j}^{(i)} - P_{\widehat{\mathcal{O}}_{j}^{(i)}}\theta^*\|_{2} > \frac{\Delta}{2} \right), \left(\|\widetilde{\theta}_{1,j}^{(i)} - P_{1}\theta^*\|_{2} > \frac{\Delta}{2}\right). \label{eqn:errorevent}
\end{align}
The above implication follows from the contrapositive argument by showing that that the negation of both the events in (\ref{eqn:errorevent}) must be simultaneously true so that $\widehat{\mathcal{O}}_{j}^{(i)} = 1$ holds. The contrapositive argument can be proved as follows: observe that for any $k (\neq 1) \in S_{j}^{(i)}$, $\left(\|\widetilde{\theta}_{k,j}^{(i)} - P_{k}\theta^*\|_{2} \leq \frac{\Delta}{2} \right) \cap \left(\|\widetilde{\theta}_{1,j}^{(i)} - P_{1}\theta^*\|_{2} \leq \frac{\Delta}{2}\right)$ implies $\|P_{k}\theta^*\|_{2} - \frac{\Delta}{2} \leq \|\widetilde{\theta}_{k,j}^{(i)}\|_{2} \leq \|P_{k}\theta^*\|_{2} + \frac{\Delta}{2}$ and $\|P_{1}\theta^*\|_{2} - \frac{\Delta}{2} \leq \|\widetilde{\theta}_{1,j}^{(i)}\|_{2} \leq \|P_{1}\theta^*\|_{2} + \frac{\Delta}{2}$, which follows from the triangle inequality. Now, suppose that $\|\widetilde{\theta}_{k,j}^{(i)}\|_{2} \geq \|\widetilde{\theta}_{1,j}^{(i)}\|_{2}$. Using the bounds on $\|\widetilde{\theta}_{k,j}^{(i)}\|_{2}$ and $\|\widetilde{\theta}_{1,j}^{(i)}\|_{2}$ obtained from the triangle inequality, $\|\widetilde{\theta}_{k,j}^{(i)}\|_{2} \geq \|\widetilde{\theta}_{1,j}^{(i)}\|_{2}$ implies that $\|P_{k}\theta^*\|_{2} - \frac{\Delta}{2} \geq \|P_{1}\theta^*\|_{2} + \frac{\Delta}{2}$, which after rearranging the terms results in the following inequality: $\|P_{1}\theta^*\|_{2} - \|P_{k}\theta^*\|_{2} \leq -\Delta$. Notice that the left hand side of this inequality is strictly positive, as $P_1 \theta^* = \theta^*$. This is a contradiction, as a strictly positive number cannot be less than a strictly negative number, as $\Delta>0$. Therefore, our initial assertion that $\|\widetilde{\theta}_{k,j}^{(i)}\|_{2} \geq \|\widetilde{\theta}_{1,j}^{(i)}\|_{2}$ is incorrect and the claim in (\ref{eqn:errorevent}) follows.

The above discussion results in the following lemma:

\begin{lemma}
\label{lemma:wrongsubspace}
For any agent $i \in [N]$ and phase $j \geq \tau_0$, we have
\begin{align*}
\mathbb{P}\left(1 \in S_{j}^{(i)}, \widehat{\mathcal{O}}_{j}^{(i)} \neq 1\right) \leq
4m \ \exp\left(-\frac{\Delta^2}{m}b^{\frac{j-1}{2}}\right). 
\end{align*}
\end{lemma}
\begin{proof}
We have
\begin{align*}
\mathbb{P}\left(1 \in S_{j}^{(i)}, \widehat{\mathcal{O}}_{j}^{(i)} \neq 1\right) &= \mathbb{P}\left(1 \in S_{j}^{(i)}, \|\widetilde{\theta}_{\widehat{\mathcal{O}}_{j}^{(i)},j}^{(i)}\|_{2} \geq \|\widetilde{\theta}_{k,j}^{(i)}\|_{2}\;\textrm{for all } k \in S_{j}^{(i)}\right) \\
&\leq \mathbb{P}\left(1 \in S_{j}^{(i)},  \|\widetilde{\theta}_{\widehat{\mathcal{O}}_{j}^{(i)},j}^{(i)}\|_{2} \geq \|\widetilde{\theta}_{1,j}^{(i)}\|_{2}\right)\\ 
&\stackrel{(a)}{\leq}\mathbb{P}\left(\left(\|\widetilde{\theta}_{\widehat{\mathcal{O}}_{j}^{(i)},j}^{(i)} - P_{\widehat{\mathcal{O}}_{j}^{(i)}}\theta^*\|_{2} > \frac{\Delta}{2} \right) \cup \left(\|\widetilde{\theta}_{1,j}^{(i)} - P_{1}\theta^*\|_{2} > \frac{\Delta}{2}\right)\right)\\
&\stackrel{(b)}{\leq}\mathbb{P}\left(\|\widetilde{\theta}_{\widehat{\mathcal{O}}_{j}^{(i)},j}^{(i)} - P_{\widehat{\mathcal{O}}_{j}^{(i)}}\theta^*\|_{2} > \frac{\Delta}{2} \right) + \mathbb{P}\left(\|\widetilde{\theta}_{1,j}^{(i)} - P_{1}\theta^*\|_{2} > \frac{\Delta}{2}\right)\\
&\stackrel{(c)}{\leq} 4m \ \exp\left(-\frac{\Delta^2}{m}b^{\frac{j-1}{2}}\right),
\end{align*}
where step $(a)$ follows from the fact that $\{\widehat{\mathcal{O}}_{j}^{(i)} \neq 1\}$ implies at least one of the events in (\ref{eqn:errorevent}) must be true, step $(b)$ follows from union bound, and step $(c)$ follows from Lemma \ref{lemma:intstep}. This concludes the proof of Lemma \ref{lemma:wrongsubspace}.
\end{proof}

\begin{prop}
\label{prop:reg_btau}
	The freezing time $\tau + \frac{b^{\tau}-1}{b-1}$ is bounded by
	\begin{align*}
\mathbb{E}\left[\tau + \frac{b^{\tau}-1}{b-1}\right] \leq  g(b)\left(\lceil b^{2\tau_0} \rceil + \frac{48b^3}{\log b}.\frac{m^{4}N}{\Delta^6}+b\mathbb{E}[b^{2 \widehat{\tau}_{\mathrm{spr}}}]\right),
	\end{align*}
	where $\tau_{0}$ and $g(b)$ are defined in Theorem \ref{thm:mainresult}.
	\label{prop:strong_cost}
\end{prop}
\begin{proof}
We follow similar steps as in the proof for Proposition $3$ in \cite{aistats_gossip} for establishing the above result. As $\tau$ is a non-negative random variable,
\begin{align*}
	\mathbb{E}[b^{\tau}] &\leq \mathbb{E}[\lceil b^{\tau} \rceil] \\
	&= \sum_{t \geq 1} \mathbb{P}(\lceil b^{\tau} \rceil \geq t) \\
	&\leq 1+ \sum_{t \geq 2} \mathbb{P}(b^{\tau}+1 \geq t)\\
	&\leq 1+\sum_{t \geq 2} \mathbb{P}\left(\tau \geq \lfloor \log_{b}(t-1) \rfloor\right) \\
	&= 1+\sum_{t \geq 1} \mathbb{P}\left(\widehat{\tau}_{\mathrm{stab}} + \widehat{\tau}_{\mathrm{spr}} \geq  \lfloor \log_{b}t \rfloor\right) \\
	&\leq 1+\sum_{t \geq 1} \mathbb{P}\left(\widehat{\tau}_{\mathrm{stab}}  \geq \frac{1}{2}\lfloor \log_{b}t \rfloor \right) + \sum_{t \geq 1} \mathbb{P} \left(\widehat{\tau}_{\mathrm{spr}} \geq \frac{1}{2} \lfloor \log_{b}t \rfloor\right)\\
	&\leq 1+\sum_{t \geq 1} \mathbb{P}\left(\widehat{\tau}_{\mathrm{stab}}  \geq \frac{1}{2}\lfloor \log_{b}t \rfloor \right) + \sum_{t \geq 1} \mathbb{P} \left(\widehat{\tau}_{\mathrm{spr}} \geq \frac{1}{2}\log_{b}t -\frac{1}{2}\right)\\
	&= 1+\sum_{t \geq 1} \mathbb{P}\left(\widehat{\tau}_{\mathrm{stab}}  \geq \frac{1}{2}\lfloor \log_{b}t \rfloor \right) + \sum_{t \geq 1} \mathbb{P} \left(b^{2\widehat{\tau}_{\mathrm{spr}}+1} \geq t\right)\\
	&\leq 1+\lceil b^{2\tau_0} \rceil + \sum_{t \geq \lceil b^{2\tau_0} \rceil+1}  \mathbb{P}\left(  \widehat{\tau}_{\mathrm{stab}} \geq \frac{1}{2}\lfloor \log_{b}t \rfloor \right) + b\mathbb{E}[b^{2 \widehat{\tau}_{\mathrm{spr}}}].
	\end{align*}
	Since the spreading time with the standard rumor model dominates $\widehat{\tau}_{\mathrm{spr}}$, we use this to bound the term $\mathbb{E}[b^{2 \widehat{\tau}_{\mathrm{spr}}}]$ in Proposition \ref{prop:couple_spread} after the proof of Proposition \ref{prop:reg_btau}. The summation in the last step is bounded by using Lemma \ref{lemma:wrongsubspace}, as follows: for some fixed $x \geq \tau_0$, we have
	\begin{align}
	\mathbb{P}(\widehat{\tau}_{\mathrm{stab}} \geq x) &= \mathbb{P}\left(\bigcup_{i =1}^{N}(\widehat{\tau}_{\mathrm{stab}}^{(i)} \geq x)\right) \nonumber\\
	&\leq \sum_{i=1}^{N} \mathbb{P}(\widehat{\tau}_{\mathrm{stab}}^{(i)} \geq x),\nonumber\\
	&= \sum_{i=1}^{N} \mathbb{P} \left(\bigcup_{l=x}^{\infty} (\chi_l^{(i)} = 1) \right)\nonumber\\
	&\leq \sum_{i=1}^{N}  \sum_{l \geq x} \mathbb{P}\left(  \chi_l^{(i)} = 1 \right)\nonumber\\
	&= \sum_{i=1}^{N}  \sum_{l \geq x} \mathbb{P}\left(1 \in S_{j}^{(i)}, \widehat{\mathcal{O}}_{j}^{(i)} \neq 1\right)\nonumber\\
	&\stackrel{(a)}{\leq} \sum_{i=1}^{N}  \sum_{l \geq x} 4m \ \exp\left(-\frac{\Delta^2}{m} b^{\frac{l-1}{2}}\right)\nonumber\\
	&=  \sum_{l \geq x} 4mN \ \exp\left(-\frac{\Delta^2}{m} b^{\frac{l-1}{2}}\right) \label{eq:taustabtail},
	\end{align}
	where we use Lemma \ref{lemma:wrongsubspace} in step $(a)$. Thus, we obtain the following:
		\begin{align*}
	\sum_{t \geq 2^{2\tau_0}}  \mathbb{P}\left(\widehat{\tau}_{\mathrm{stab}} \geq \frac{1}{2}\lfloor \log_{b}t \rfloor \right) 
	&\leq \sum_{t \geq \lceil b^{2\tau_0} \rceil+1}   \sum_{l \geq \frac{1}{2}\lfloor \log_{b}t \rfloor} 4mN \ \exp\left(-\frac{\Delta^2}{m} b^{\frac{l-1}{2}}\right)\\
    &\leq  4mN\sum_{t \geq \lceil b^{2\tau_0} \rceil + 1}   \sum_{l \geq \frac{1}{2}\lfloor \log_{b}t \rfloor} \exp\left(-\frac{\Delta^2}{m} b^{\frac{l-1}{2}}\right)\\
	&\stackrel{(b)}{\leq} 4mN\sum_{l \geq \tau_0} \sum_{t = \lceil b^{2\tau_0} \rceil +1}^{b^{2l+1}} \exp\left(-\frac{\Delta^2}{m} b^{\frac{l-1}{2}}\right)\\
	&\leq 4mN\sum_{l \geq \tau_0} b^{2l+1} \exp\left(-\frac{\Delta^2}{m} b^{\frac{l-1}{2}}\right),
	\end{align*}
	where step $(b)$ follows by re-writing the range of summations. The sum $\sum_{l \geq \tau_0} b^{2l+1} \exp\left(-\frac{\Delta^2}{m} b^{\frac{l-1}{2}}\right)$ is bounded as follows:
	\begin{align*}
	    \sum_{l \geq \tau_0} b^{2l+1} \exp\left(-\frac{\Delta^2}{m} b^{\frac{l-1}{2}}\right) 
	    &\leq \int_{x=1}^{\infty} b^{2x+1} \exp\left(-\frac{\Delta^2}{m}b^{\frac{x-1}{2}}\right) \;\mathrm{d}x\\
	    &\stackrel{(c)}{\leq} \int_{u=1}^{\infty} b^{4\log_{b}u+3} \exp\left(-\frac{\Delta^2}{m}u\right) \;\frac{2}{u \log b}\mathrm{d}u\\
	    &\leq \frac{2b^3}{\log b} \int_{u=0}^{\infty} u^3 \exp\left(-\frac{\Delta^2}{m}u\right) \;\mathrm{d}u\\
	    &= \frac{12b^3}{\log b}\left(\frac{m}{\Delta^2}\right)^3,
	\end{align*}
	where we perform change of variables with $x = 2\log_{b}u + 1$ in step $(c)$. Therefore, 
	\begin{equation*}
	    \mathbb{E}\left[b^{\tau}-1\right] \leq \lceil b^{2\tau_0} \rceil + \frac{48b^3}{\log b}.\frac{m^{4}N}{\Delta^6}+b\mathbb{E}[b^{2 \widehat{\tau}_{\mathrm{spr}}}].
	\end{equation*}
	
	For bounding $\mathbb{E}[\tau]$, notice that $\tau \leq \frac{b^{\tau}-1}{\log b}$, where we use the fact that for all $b>1$, $b^x - x \log b - 1 \geq 0$ for all $x \geq 0$. Thus, $\mathbb{E}[\tau] \leq \frac{1}{\log b} \mathbb{E}[b^{\tau}-1]$. Substituting the bound for $\mathbb{E}[b^{\tau}-1]$ obtained above completes the proof of Proposition \ref{prop:reg_btau}.
	\end{proof}

\begin{prop}	
\label{prop:couple_spread}
	The random variable $\widehat{\tau}_{\mathrm{spr}}$ is stochastically dominated by $\tau_{\mathrm{spr}}^{(G)}$.
\end{prop}
\begin{proof}
The proof follows in a similar way as the proof for Proposition $4$ in \cite{aistats_gossip}.
\end{proof}

\begin{prop}
\label{prop:tau0bound}
$\tau_0$ defined in Theorem \ref{thm:mainresult} is bounded by
\begin{align*}
    \tau_0 \leq 2 \log_{b}\left(16m\left(\frac{K}{N}+2\right)\right)+1.
\end{align*}
\end{prop}
\begin{proof}
From Theorem \ref{thm:mainresult},
\begin{align*}
    \tau_0 = \min\left\{j \in \mathbb{N}: \forall j' \geq j, \lceil b^{j'-1} \rceil \geq 8m\left(\frac{K}{N}+2\right)\lceil b^{\frac{j'-1}{2}} \rceil\right\}.
\end{align*}
As $8m\left(\frac{K}{N}+2\right)\lceil b^{\frac{j'-1}{2}} \rceil \leq 8m\left(\frac{K}{N}+2\right) b^{\frac{j'-1}{2}} + 8m\left(\frac{K}{N}+2\right) \leq 16m\left(\frac{K}{N}+2\right) b^{\frac{j'-1}{2}}$, the minimum value of $j$ that satisfies $b^{j-1} \geq 16m\left(\frac{K}{N}+2\right) b^{\frac{j-1}{2}}$ is an upper bound on $\tau_0$. Rearranging the terms results in $j \geq 2 \log_{b}\left(16m\left(\frac{K}{N}+2\right)\right)+1$ and thus, $\tau_0 \leq 2 \log_{b}\left(16m\left(\frac{K}{N}+2\right)\right)+1$.
\end{proof}

\subsection{Proof of Theorem \ref{thm:mainresult}}
From Theorem \ref{th:regprojlinucb} and $\delta = \frac{1}{T}$, $\mathbb{E}[R_{\mathrm{proj},T}]$ is bounded as follows:
\begin{multline}
    \mathbb{E}[R_{\mathrm{proj},T}] = \mathbb{E}\left[R_{\mathrm{proj},T} \ \mathbf{1}\left(R_{\mathrm{proj},T} \leq \sqrt{8mT\beta_{T, \delta}^{2} \log \left(1+\frac{T}{m\lambda}\right)}\right)\right] \\
    + \mathbb{E}\left[R_{\mathrm{proj},T} \ \mathbf{1}\left(R_{\mathrm{proj},T} > \sqrt{8mT\beta_{T, \delta}^{2} \log \left(1+\frac{T}{m\lambda}\right)}\right)\right]\nonumber
    \end{multline}
    \begin{multline}
    \stackrel{(a)}{\leq} \sqrt{8mT\beta_{T, \delta}^{2} \log \left(1+\frac{T}{m\lambda}\right)} + 2S \mathbb{E}\left[\sum_{t = 1}^{T} \mathbf{1}\left(R_{\mathrm{proj},T} > \sqrt{8mT\beta_{T, \delta}^{2} \log \left(1+\frac{T}{m\lambda}\right)}\right)\right],\nonumber
    \end{multline}
    \begin{multline}
    = \sqrt{8mT\beta_{T, \delta}^{2} \log \left(1+\frac{T}{m\lambda}\right)} + 2S \sum_{t = 1}^{T} \mathbb{E}\left[\mathbf{1}\left(R_{\mathrm{proj},T} > \sqrt{8mT\beta_{T, \delta}^{2} \log \left(1+\frac{T}{m\lambda}\right)}\right)\right],\nonumber
    \end{multline}
    \begin{align}
\mathbb{E}[R_{\mathrm{proj},T}] &\stackrel{(b)}{\leq} \sqrt{8mT\beta_{T}^{2} \log \left(1+\frac{T}{m\lambda}\right)} + 2S.\frac{1}{T}.\sum_{t = 1}^{T} 1, \nonumber\\
    &= \sqrt{8mT\beta_{T}^{2} \log \left(1+\frac{T}{m\lambda}\right)} + 2S.\label{eqn:expregproj}
    \end{align}
In step $(a)$, the first sum is trivially bounded by $\sqrt{8mT\beta_{T}^{2} \log \left(1+\frac{T}{m\lambda}\right)}$ and the second sum uses the definition of $R_{\mathrm{proj},T}$ with $w_t \leq 2S$ for all $t \in \mathbb{N}$. Step $(b)$ uses Theorem \ref{th:regprojlinucb} with $\delta = \frac{1}{T}$.
\\ 

Substituting the results of Propositions \ref{prop:strong_cost} and \ref{prop:couple_spread}, along with (\ref{eqn:expregproj}) into Proposition \ref{prop:regdecomp} concludes the proof of Theorem \ref{thm:mainresult}.

\section{Regret Upper Bound for Single Agent Running \name \ Algorithm Without Communications}
\label{proof:main_single}
\begin{theorem}
\label{thm:mainresult_single}
With the same assumptions as in Theorem \ref{thm:mainresult}, when a single agent runs \name \ Algorithm in case of no communication, the regret after any time $T \in \mathbb{N}$ is bounded by
\begin{multline}
	\mathbb{E}[R_T] \leq  \underbrace{\sqrt{8mT\beta_{T}^{2}\log\left(1+\frac{T}{m\lambda}\right)}+2S}_{\text{{Projected LinUCB Regret}}} + \underbrace{2Sg(b)\left(\lceil b(16mK)^2 \rceil + \frac{8b^2}{\log b}.\frac{m^2}{\Delta^2}\right)}_{\text{{Constant Cost of Right Subspace Search}}} \\
	+ \underbrace{16 mKS \log_{b}(h_{b, T}) + 16 mKS \frac{\sqrt{h_{b, T}}-1}{\sqrt{b}-1}}_{\text{Cost of subspace exploration}}. \label{eqn:single-agent-regret-thm}
	\end{multline}
	Here, $\beta_T$, $g(b)$, and $h_{b, T}$ are the same as in Theorem \ref{thm:mainresult}.
\end{theorem}

\begin{proof}
Before we prove Theorem \ref{thm:mainresult_single}, we set some notation. 
Let $\widehat{\mathcal{O}}_{j} = \arg \max_{k \in [K]} \|\widetilde{\theta}_{k,j}^{(i)}\|_{2}$, i.e., $\widehat{\mathcal{O}}_{j}$ denotes the ID of subspace in which the single agent plays Projected LinUCB in the exploit time slots of phase $j$. Let $\bar{\tau}_{0} = \min\left\{j \in \mathbb{N}: \forall j' \geq j, \lceil b^{j'-1} \rceil \geq 8mK\lceil b^{\frac{j'-1}{2}} \rceil\right\}$. Following the proof of Proposition \ref{prop:tau0bound}, it can be shown that $\bar{\tau}_0 \leq 2 \log_{b}(16mK)+1$.
\\

We also define a random phase $\tau_{\mathrm{freeze}}$ as follows:
\begin{align*}
    \tau_{\mathrm{freeze}} = \inf\{j \geq \bar{\tau}_{0}: \forall j' \geq j, \widehat{\mathcal{O}}_{j'}=1\}.
\end{align*}
Thus, $\tau_{\mathrm{freeze}}$ is the earliest phase after which the single agent will play the projected LinUCB from the subspace $\mathrm{span}(U_1)$ in the exploit time slots of a phase. Notice that the random phase $\tau_{\mathrm{freeze}}$ plays the same role as $\widehat{\tau}_{stab}$ in the multi-agent case. This suggests that the regret analysis in this case must follow the same chain of argument as for the multi-agent case.
\\

Following the same steps as for Proposition \ref{prop:regdecomp}, the bound on the regret of the single agent after $T$ time steps is given by
\begin{multline}
    \mathbb{E}[R_{T}] \leq 2S\left(\mathbb{E}\left[\tau_{\mathrm{freeze}} + \frac{b^{\tau_{\mathrm{freeze}}}-1}{b-1}\right]\right) + \mathbb{E}[R_{\mathrm{proj},T}]+16 mKS \log_{b}(h_{b, T}) \\ + 16 mKS \frac{\sqrt{h_{b, T}}-1}{\sqrt{b}-1}. \label{eq:regdecompsingle}
\end{multline}
We have already shown in (\ref{eqn:expregproj}) that 
\begin{align*}
    \mathbb{E}[R_{\mathrm{proj}, T}] \leq \sqrt{8mT\beta_{T}^{2} \log \left(1+\frac{T}{m\lambda}\right)} + 2S.
\end{align*}

We will now bound $\mathbb{E}[\tau_{\mathrm{freeze}}]$ and $\mathbb{E}[b^{\tau_{\mathrm{freeze}}}]$ to complete the proof. We first bound $\mathbb{E}[b^{\tau_{\mathrm{freeze}}}]$. From the definition of expectation for positive random variables,
\begin{align*}
    \mathbb{E}[b^{\tau_{\mathrm{freeze}}}] &\leq \mathbb{E}[\lceil b^{\tau_{\mathrm{freeze}}} \rceil]\\
    &= \sum_{t=1}^{\infty} \mathbb{P}(\lceil b^{\tau_{\mathrm{freeze}}} \rceil \geq t)\\
    &\leq 1 + \sum_{t=2}^{\infty}\mathbb{P}(b^{\tau_{\mathrm{freeze}}}+1 \geq t)\\
    &\leq 1 + \sum_{t=1}^{\infty}\mathbb{P}(\tau_{\mathrm{freeze}} \geq \lfloor \log_{b} t \rfloor)\\
    &\leq 1+\lceil b^{\bar{\tau}_0} \rceil+ \sum_{t=\lceil b^{\bar{\tau}_0} \rceil+1}^{\infty} \mathbb{P}(\tau_{\mathrm{freeze}} \geq \lfloor \log_{b}t \rfloor)\\
    \end{align*}
    We will bound the summation in the last term below:
    \begin{align*}
    \sum_{t=\lceil b^{\bar{\tau}_0} \rceil+1}^{\infty} \mathbb{P}(\tau_{\mathrm{freeze}} \geq \lfloor \log_{b}t \rfloor) &\stackrel{(a)}{=} \sum_{t=\lceil b^{\bar{\tau}_0} \rceil+1}^{\infty} \mathbb{P}\left(\bigcup_{j \geq \lfloor \log_{b}t \rfloor} (\widehat{\mathcal{O}}_{j} \neq 1)\right)\\
    &\leq \sum_{t=\lceil b^{\bar{\tau}_0} \rceil+1}^{\infty} \sum_{j \geq \lfloor \log_{b}t \rfloor}\mathbb{P}(\widehat{\mathcal{O}}_{j} \neq 1)\\
    &\stackrel{(b)}{\leq} \sum_{t=\lceil b^{\bar{\tau}_0} \rceil+1}^{\infty} \sum_{j \geq \lfloor \log_{b}t \rfloor} 4m \exp\left(-\frac{\Delta^2}{m}.b^{\frac{j-1}{2}}\right)\\
    &\stackrel{(c)}{\leq} 4m \sum_{j \geq \bar{\tau}_0 }\sum_{t=\lceil b^{\bar{\tau}_0} \rceil+1}^{b^{j+1}} \exp\left(-\frac{\Delta^2}{m}.b^{\frac{j-1}{2}}\right)\\
    &\leq 4m \sum_{j \geq \bar{\tau}_0} b^{j+1} \exp\left(-\frac{\Delta^2}{m}.b^{\frac{j-1}{2}}\right)\\
    &\leq 4m \int_{x = 1}^{\infty} b^{x+1} \exp\left(-\frac{\Delta^2}{m}.b^{\frac{x-1}{2}}\right)\;\mathrm{d}x\\
    &\stackrel{(d)}{=} \frac{8mb^2}{\log b} \int_{u = 1}^{\infty} u  \exp\left(-\frac{\Delta^2}{m}.u\right)\;\mathrm{d}u\\
    &\leq \frac{8mb^2}{\log b} \int_{u = 0}^{\infty} u  \exp\left(-\frac{\Delta^2}{m}.u\right)\;\mathrm{d}u\\
    &\leq \frac{8b^2}{\log b}.\frac{m^2}{\Delta^{2}}.
\end{align*}
We use the definition of $\tau_{\mathrm{freeze}}$ in Step $(a)$. In step $(b)$, we substitute the bound on probability of choosing the subspace other than $\mathrm{span}(U_1)$ from Lemma \ref{lemma:wrongsubspace}. We interchange the order of summation in step $(c)$. In step $(d)$, we perform a change of variables with $x=2\log_{b}u + 1$. Therefore,
\begin{equation*}
    \mathbb{E}\left[b^{\tau_{\mathrm{freeze}}}-1\right] \leq \lceil b(16mK)^2 \rceil + \frac{8b^2}{\log b}.\frac{m^2}{\Delta^2}. 
\end{equation*}

For bounding $\mathbb{E}[\tau_{\mathrm{freeze}}]$, notice that $\tau_{\mathrm{freeze}} \leq \frac{b^{\tau_{\mathrm{freeze}}}-1}{\log b}$, where we use the fact that for all $b > 1$, $b^x - x \log b - 1 \geq 0$ for all $x \geq 0$. Thus, $\mathbb{E}[\tau_{\mathrm{freeze}}] \leq \frac{1}{\log b} \mathbb{E}[b^{\tau_{\mathrm{freeze}}}-1]$. Substituting the bound for $\mathbb{E}[b^{\tau_{\mathrm{freeze}}}-1]$ obtained above completes the proof of Theorem \ref{thm:mainresult_single}.
\end{proof}

\section{Analysis of Projected LinUCB}
\label{proof:projlinucb}
We now analyze Projected LinUCB as a separate black box, where every agent plays from the correct subspace containing $\theta^*$ for $T$ steps. This holds as every agent plays Projected LinUCB on the subspace $\mathrm{span}(U_1)$ during the exploit time slots after the freezing phase $\tau$ by only using the Projected LinUCB actions and rewards for $\mathrm{span}(U_1)$, which of course will happen for less than $T$ steps.
As this analysis is valid for all agents, we drop the superscript $(i)$ from all the pertinent variables, where $i \in [N]$. Recall that $\bar{V}_{t}(\lambda) = P_1 (V_{t-1} + \lambda I_{d}) P_1 = U_1 \Sigma_{t} U_{1}^{T}$ where $\Sigma_{t} = U_{1}^{T} V_{t-1} U_{1} + \lambda I_{m}$. Define $\Tilde{V}_{t}(\lambda) = \bar{V}_{t}(\lambda) - \lambda P_1 = P_1 V_{t-1} P_1 = U_1 \bar{\Sigma}_{t} U_{1}^{T}$, where $\bar{\Sigma}_{t} = \Sigma_{t} - \lambda I_{m} = U_{1}^{T} V_{t-1} U_{1}$.

\subsection{Confidence Set Construction and Analysis}
The construction of the confidence set is done in a similar way as in \cite{hassibi_PSLB, lattimore_book}. 

\begin{theorem}
\label{ConfidenceSet}
Let $\delta \in (0, 1)$ and $\beta_{t, \delta} = S\sqrt{\lambda} + \sqrt{2 \log\frac{1}{\delta} + m \log \left(1 + \frac{t-1}{\lambda m}\right)}$. Then,
\begin{align*}
\mathbb{P}\left(\forall t \in \mathbb{N}, \theta^{*} \in \mathcal{C}_{t}^{(i)}\right) \geq 1-\delta,
\end{align*}
where
\begin{align}
    \mathcal{C}_{t}^{(i)} = \left\{\theta \in \mathbb{R}^{d}: ||\widehat{\theta}_{t}^{(i)}-\theta||_{\bar{V}_{t}(\lambda)} \leq \beta_{t,\delta}\right\}.
    \label{eqn:confidence_set}
\end{align}
\end{theorem}
The proof of Theorem \ref{ConfidenceSet} is adapted from the proof of Theorem $8$ in \cite{hassibi_PSLB}, which considers a different setting: the subspace in which $\theta^{*}$ lies is unknown and needs to be estimated. The error in estimating the correct subspace appears in the construction of projected confidence set in Theorem $8$ of \cite{hassibi_PSLB}. However, in our setting, agents are aware of the true projection matrices of the subspaces and thus do not need to account for subspace estimation error while constructing confidence sets.
This necessitates a different definition of the confidence set given in (\ref{eqn:confidence_set}).

Thus, when agents play Projected LinUCB on the correct subspace, they don't have to pay the overhead of recovering the actual subspace from the perturbed action vectors. Hence, Theorem \ref{ConfidenceSet} is proved by substituting the estimated projection matrix $\hat{P}_{t}$ equal to the true projection matrix $P_1$ in the proof of Theorem $8$ in \cite{hassibi_PSLB} for all $t \in \mathbb{N}$. 

\subsection{Regret Analysis}
Before bounding the regret, let us set some notation here. We have 
\begin{align*}
    a_{t} \in \arg \max_{a \in \mathcal{A}_{t}} \max_{\theta \in \mathcal{C}_{t}^{(i)}} \langle \theta, P_{1}a \rangle.
\end{align*}
Let $a_{t}^{*} = \arg \max_{a \in \mathcal{A}_{t}} \langle \theta^{*}, a \rangle$ and $\bar{\theta}_{t} \in \mathcal{C}_{t}^{(i)}$ such that $\bar{\theta}_{t} = \arg \max_{\theta \in \mathcal{C}_{t}^{(i)}} \langle \theta, P_{1}a_{t} \rangle$. The following theorem characterizes the regret after every agent has the right subspace.

\begin{theorem}
\label{th:regprojlinucb}
With probability at least $1 - \delta$, the regret incurred after playing Projected LinUCB on the subspace containing $\theta^*$ for $T$ steps satisfies
$$R_{\mathrm{proj}, T} \leq \sqrt{8mT\beta_{T, \delta}^{2} \log \left(1+\frac{T}{m\lambda}\right)}.$$
\end{theorem}
\begin{proof}
The proof of Theorem \ref{th:regprojlinucb} is contingent on the following lemma. A similar lemma appears as Lemma $19.4$ in \cite{lattimore_book}.
\begin{lemma}
\label{instregbound}
Let $a_{1}, \cdots, a_{T}$ be the sequence of action vectors played up to and including time $T$. Then,
\begin{align*}
\sum_{t = 1}^{T} \min\left\{1, ||a_{t}||_{\bar{V}_{t}(\lambda)^{\dagger}}^{2}\right\} \leq 2 \log\left(\frac{\det(\Sigma_{T+1})}{\det(\Sigma_{1})}\right).
\end{align*}
\end{lemma}
\begin{proof}
The proof is identical to the proof of \cite{lattimore_book}, Lemma $19.4$, except that we use the recursive update of $\det(\Sigma_{t})$ instead of $\det(V_{t-1})$.
\end{proof}
\noindent We now have the required ingredients to complete the proof of Theorem \ref{th:regprojlinucb}. Using the fact that $\theta^{*} \in \mathcal{C}_{t}^{(i)}$ and from the algorithm definitions, the following chain of inequalities is true: 
\begin{align}
\langle \theta^{*}, a_{t}^{*} \rangle = \langle P_{1}\theta^{*}, a_{t}^{*} \rangle = \langle \theta^{*}, P_{1}a_{t}^{*} \rangle \leq \max_{\theta \in \mathcal{C}_{t}^{(i)}} \langle \theta, P_{1}a_{t}^{*} \rangle \leq \langle \bar{\theta}_{t}, P_{1}a_{t} \rangle. \label{eq:optimisticest}
\end{align}
Thus, for all $t \in \mathbb{N}$, 
\begin{align}
w_{t}^{(i)} &\leq 2 \beta_{t, \delta} ||a_{t}||_{\bar{V}_{t}(\lambda)^{\dagger}} \label{eq:instregbound2}, 
\end{align}
which is shown in a similar way as bounding $r_t$ in \cite{lattimore_book}, Theorem $19.2$. However, we additionally use the fact that $\theta^{*} = P_{1}\theta^{*}$ and $P_{1}^2 = P_{1}$. 

While proving Proposition \ref{prop:strong_freeze}, we showed that $w_{t}^{(i)} \leq 2S$ for all $t \in \mathbb{N}$. Combining this with (\ref{eq:instregbound2}) and $\beta_{T, \delta} \geq S$ results in 
\begin{align}
w_{t}^{(i)} &\leq 2 \min\left\{S, \beta_{t, \delta}||a_{t}||_{\bar{V}_{t}(\lambda)^{\dagger}}\right\}\nonumber\\
&\leq 2 \beta_{T, \delta} \min\left\{1, ||a_{t}||_{\bar{V}_{t}(\lambda)^{\dagger}}\right\}. \label{eq: instregbound3}
\end{align}
Therefore, the cumulative regret incurred by playing projected LinUCB for $T$ can be bounded as follows: let $\mathbb{I}$ denote an all-ones column vector of size $T$, and $\mathbf{R}$ be a column vector containing the elements $w_{1}, \cdots, w_{T}$. Then,
\begin{align*}
R_{\mathrm{proj}, T} &=\sum_{t = 1}^{T} w_{t}^{(i)} \\
&= \mathbb{I}^{T}\mathbf{R}\\
&\stackrel{(a)}{\leq} \sqrt{T} \sqrt{\sum_{t = 1}^{T} (w_{t}^{(i)})^{2}}\\
&\stackrel{(b)}{\leq} \sqrt{T} \sqrt{4 \beta_{T, \delta}^{2}\sum_{t = 1}^{T} \min\left\{1, ||A_{t}||_{\bar{V}_{t}(\lambda)^{\dagger}}^{2}\right\}}\\
&\stackrel{(c)}{\leq}\sqrt{8T\beta_{T, \delta}^{2}\log\left(\frac{\det(\Sigma_{T+1})}{\det(\Sigma_1)}\right)}\\
&\stackrel{(d)}{\leq}\sqrt{8mT\beta_{T, \delta}^{2}\log\left(1+\frac{T}{m\lambda}\right)}.
\end{align*}
Step $(a)$ follows from Cauchy-Schwarz inequality. In step $(b)$, we use (\ref{eq: instregbound3}). Step $(c)$ is obtained by application of Lemma \ref{instregbound}. 
Step $(d)$ results from the fact that $\det(\Sigma_{1}) = \det(\Lambda) = \lambda^{m}$ and $\det(\Sigma_t) \leq \left(\lambda + \frac{t-1}{m}\right)^{m}$ (which follows from Lemma $11$ in \cite{hassibi_PSLB}), thus completing the proof.
\end{proof}


\end{appendices}

\end{document}